\documentclass[onecolumn]{IEEEtran}       

\usepackage{amsthm}

\usepackage[pdftex]{graphicx}
\usepackage[utf8]{inputenc}
\usepackage[margin=2cm]{geometry}
\usepackage{pifont}
\usepackage{enumitem}
\usepackage{multirow}
\usepackage{cite}
\usepackage{xr}
\usepackage{xcite}
\usepackage{amsmath,amssymb,amsfonts}
\usepackage{algorithm}
\usepackage{algorithmic}
\usepackage{textcomp}
\usepackage{xcolor}
\usepackage{caption}
\usepackage{subcaption}
\usepackage{salt}

\newcommand{\Rn}{\mathbb{R}}
\newcommand{\Pn}{\mathbb{P}}
\newcommand{\xs}{\x^\star}
\newcommand{\vls}{\lam^\star}
\newcommand{\vlth}{\lam^\theta}
\newcommand{\xt}{\tilde{\x}}
\newcommand{\lt}{\tilde{\lam}}

\usepackage{tabularx,ragged2e,booktabs,caption}
\usepackage{bm}

\newtheorem{thm}{Theorem}
\newtheorem{lemma}{Lemma}

\newtheorem{assumption}{}

\newcommand{\px}[1]{\Pi_{\cX}\left\{#1\right\}}
\newcommand{\po}[1]{\left[#1\right]_{+}}

\newcolumntype{C}[1]{>{\Centering}m{#1}}

\newcommand{\lam}{\bm{\lambda}}
\newcommand{\lamb}{\lambda}
\newcommand{\ls}{\lamb^{\star}}

\usepackage{bbm}
\newcommand{\indic}{\mathbbm{1}}
\newcommand{\Ec}[1]{\bE\bs{{#1} \cond \cF_t}}

\def \cX {{\mathcal X}}
\def \cA {{\mathcal A}}
\def \cI {{\mathcal I}}
\def \cJ {{\mathcal J}}
\def \O {{\mathcal O}}
\def \cK {{\mathcal{K}}}
\def \h {{\mathbf{h}}}
\def \g {{\mathbf{g}}}
\def \x {{\mathbf{x}}}
\def \EE {{\mathbb{E}}}
\def \gb {{\bar{\g}}}
\def \hb {{\bar{\h}}}
\def \Lh {{\hat{\mathcal{L}}}}
\def \y {{\mathbf{y}}}
\def \w {{\mathbf{w}}}
\def \lj {{\{\ell_j\}_{j=1}^J}}
\def \Lth {{\cL^{\theta}}}
\def \lt {{\tilde{\lam}}}
\def \lts {{\tilde{\lambda}}}	
\def \T {{\mathsf{T}}}

\newcommand{\cXs}{\cX^\star}
\newcommand{\chis}{\chi^\star}

\title{Stochastic Compositional Gradient Descent under Compositional Csonstraints}

\author{Srujan~Teja ~Thomdapu,   
        Harshvardhan, 
        and~Ketan~Rajawat 
        \thanks{S. T. Thomdapu and K. Rajawat are with the Department of Electrical Engineering,
        Indian Institute of Technology Kanpur, Kanpur 208016, India (e-mail: srujant@iitk.ac.in,
        ketan@iitk.ac.in). Harshvardhan is with the School of Computer and Communication Sciences, 
        Ecole Polytechnique Federale de Lausanne, Lausanne, Switzerlanda (e-mail: harshvardhan.harshvardhan@epfl.ch)}
}

\begin{document}

\maketitle

\begin{abstract}
	This work studies constrained stochastic optimization problems where the objective and constraint functions are convex and expressed as compositions of stochastic functions. The problem arises in the context of fair classification, fair regression, and the design of queuing systems. Of particular interest is the large-scale setting where an oracle provides the stochastic gradients of the constituent functions, and the goal is to solve the problem with a minimal number of calls to the oracle. Owing to the compositional form, the stochastic gradients provided by the oracle do not yield unbiased estimates of the objective or constraint gradients. Instead, we construct approximate gradients by tracking the inner function evaluations, resulting in a quasi-gradient saddle point algorithm. We prove that the proposed algorithm is guaranteed to find the optimal and feasible solution almost surely. We further establish that the proposed algorithm requires $\mathcal{O}(1/\epsilon^4)$ data samples in order to obtain an $\epsilon$-approximate optimal point while also ensuring zero constraint violation. The result matches the sample complexity of the stochastic compositional gradient descent method for unconstrained problems and improves upon the best-known sample complexity results for the constrained settings. The efficacy of the proposed algorithm is tested on both fair classification and fair regression problems. The numerical results show that the proposed algorithm outperforms the state-of-the-art algorithms in terms of the convergence rate. 
\end{abstract}

\section{Introduction}
This work considers the following constrained stochastic problem
\begin{align} \label{mainProb}
\xs = \underset{\x \in \cX}{\arg\min} ~F(\x):=&\EE_{\zeta}[f(\EE_{\xi}[\g(\x;\xi)];\zeta)]& \tag{$\mathcal{P}$}\\
\text{s.t. } ~L_j(\x):=&\EE_{\psi}[\ell_j(\EE_{\phi}[\h(\x;\phi)];\psi)]\leq 0 & \br{ 1\leq j \leq J} \nonumber
\end{align}
where random variables $\xi$, $\zeta$, $\phi$, and $\psi$ are associated with continuous and proper closed functions $\g:\Rn^n \rightarrow \Rn^m$, $f:\Rn^m \rightarrow \Rn$, $\h:\Rn^n \rightarrow \Rn^d$, and $\ell_j:\Rn^d \rightarrow \Rn$, respectively. The optimization variable $\x$ belongs to a closed convex set $\cX \subset \Rn^n$ which is easy to project onto; examples include a box or a norm-ball. Other detailed assumptions regarding the problem structure and various assumptions will be mentioned in Sec. \ref{assump}. We assume the problem \eqref{mainProb} is feasible and has finite solutions. The distribution of random variables $(\xi, \zeta, \phi, \psi)$ is not known and \eqref{mainProb} cannot be solved in closed-form or using classical optimization algorithms. Instead, the goal is to solve \eqref{mainProb} using the independent samples $(\xi_t, \zeta_t, \phi_t, \psi_t)$, that are observed in a sequential fashion. This formulation covers a wide range of optimization problems and includes the unconstrained variants considered in \cite{wang2017stochastic, wang2017accelerating, wang2016stochastic,ghadimi2020single} as well as the constrained problems in \cite{thomdapu2019optimal,thomdapu2021QoS,akhtar2020conservative}.

Constrained optimization problems such as \eqref{mainProb} can be solved using primal, dual, and primal-dual algorithms. Primal-only methods have been widely applied to problems without functional constraints, i.e., those with simple set constraints of the form $\x \in \cX$ where $\cX$ is easy to project onto. The set-constrained version of \eqref{mainProb} was first considered in \cite{wang2017stochastic}, where a quasi-gradient approach referred to as stochastic compositional gradient descent (SCGD) was proposed. The SCGD algorithm entails running two parallel iterations: one for performing quasi-gradient steps for estimating the optimal solution $\x^\star$ and another for tracking the quantity $\E{\g\br{\x^\star;\xi}}$ using samples $\{\g(\x_t,\xi_t)\}$. However, the presence of functional stochastic constraints  in \eqref{mainProb} 	 complicates the problem rendering vanilla SCGD inapplicable. Authors in \cite{thomdapu2019optimal, thomdapu2021QoS} proposed reformulating the constrained problem in \eqref{mainProb} as an unconstrained problem, that can be solved using SCGD, by adding appropriately scaled penalty functions to the objective. Although the resulting constrained SCGD (CSCGD) algorithm is provably convergent, its overall convergence rate is worse than that of SCGD owing to the additional error incurred from minimizing the penalized objective instead of the actual objective.

Stochastic dual descent has earlier been proposed to solve constrained stochastic problems that adhere to a specific form \cite{wang2010stochastic,ribeiro2010ergodic,wang2011resource}. Stationarity assumption of a random variable has been relaxed in \cite{thomdapu2021dynamic} where, the stochastic dual descent algorithm is applied to a content placement problem. However, these methods necessitate evaluating the stochastic subgradient of the dual function in closed form, which may not generally be viable. Overall, the aforementioned limitations of both primal and dual approaches appear to be fundamental in nature and motivate us to look beyond these two classes of algorithms.

Primal-dual or saddle point approaches have been applied to solve stochastic optimization problems with functional stochastic constraints in \cite{nedic2009subgradient, koppel2015saddle, bedi2017beyond,yu2017online,bedi2019asynchronous,madavan2019subgradient,akhtar2020conservative}. However, the existing variants of primal-dual methods do not handle non-linear functions of expectations either in the objective or constraint functions. The goal of this paper is to develop a primal-dual algorithm capable of solving the general problem in  \eqref{mainProb}. Of particular interest is the Arrow-Hurwicz saddle point algorithm that makes use of an augmented Lagrangian and has been successfully applied to constrained stochastic problems in \cite{bedi2019asynchronous, akhtar2020conservative}.     

Overall, the key contributions of this work are as follows.
\begin{itemize}
	\item We develop an augmented Lagrangian saddle point method to solve \eqref{mainProb}. Since the objective and the constraints in \eqref{mainProb} are composition of expected-value functions, we make use of the quasi-stochastic gradient of the Lagrangian, along the lines of \cite{wang2017stochastic}. In order to ensure that the constraints are never violated, an appropriately tightened version of \eqref{mainProb} is considered and the tightening parameters are carefully selected. 
	\item We establish, for the first time in the context of constrained stochastic optimization, the almost sure convergence of the iterates to the optimal point using a coupled supermartingale convergence argument. Additionally, we show that the sample complexity, which is the number of calls to the stochastic gradient oracle required to ensure that the optimality gap is below $\epsilon$, is given by $\cO\br{\epsilon^{-4}}$, matching the result for unconstrained case in \cite{wang2017stochastic} and improving over existing results for constrained case in \cite{thomdapu2019optimal,thomdapu2021QoS}. 
	\item 	Finally, we show that some common classification and regression problems can be formulated as compositional constrained stochastic optimization problems. Detailed numerical results over these applications demonstrate the efficacy of the proposed algorithm.
\end{itemize}

\subsection{Related Work}
There is a rich literature on stochastic approximation methods that form the foundation of the plethora of stochastic gradient variants in existence. In the present case however, the compositional structure in \eqref{mainProb} prevents us from using classical first order methods that rely on unbiased (or at least strongly consistent) stochastic gradient approximations \cite{benveniste2012adaptive,bertsekas1989parallel,borkar2009stochastic}. As already stated, the SCGD algorithm for solving the unconstrained version of \eqref{mainProb} was first proposed in \cite{wang2017stochastic}. Alternative  and more generic formulations have likewise been considered in \cite{dentcheva2017statistical} and references therein. The corresponding finite-sum variant of
the problem has subsequently been considered in \cite{lian2017finite} and solved via the variance-reduced SCGD. Accelerated versions were later proposed in \cite{wang2017accelerating,ghadimi2020single}, where the results are  improved at the cost of additional assumptions. The SCGD algorithm has been studied for corrupted samples with Morkov noises in \cite{wang2016stochastic}. A functional variant of SCGD has recently been proposed in \cite{koppel2019controlling}. Unlike these works, a more general version of the compositional problems called $conditional$ $stochastic$ $optimization$ is considered in \cite{dai2017learning,hu2020biased} where the random variables associated with inner and outer functions are not necessarily independent. However, all of these works are not applicable to the problem in \eqref{mainProb} due to the presence of stochastic constraints.


Stochastic constraints with linear functions of sample probabilities were studied in the literature. Most relevant to the current setting, stochastic dual-descent algorithm \cite{wang2011resource,ribeiro2010ergodic,bedi2018asynchronous,chen2017stochastic} and the stochastic variant of the Arrow–Hurwicz saddle point method \cite{bedi2019asynchronous}.  It has recently been shown that conservative stochastic optimization algorithm (CSOA) proposed in \cite{akhtar2020conservative} achieves a convergence rate that is the same for projected SGD. The current work generalizes the setting in \cite{akhtar2020conservative} by allowing compositional forms in both objective and constraints, thereby subsuming it. Other formulations such as those in \cite{wang2016two,chen2019learning} have also considered expectation constraints. All of these algorithms were analyzed under specific assumptions on the structure of the problem. Further, these algorithms cannot be directly applied to problems involving
compositional forms. The constrained stochastic optimization problem containing non-linear functions of expectation has recently been studied in \cite{thomdapu2019optimal} via the CSCGD algorithm. The sample complexity analysis of CSCGD establishes that after $T$ number of random variables are revealed, the algorithm converges with the rate of $\cO\br{T^{-1/12}}$. Later on the convergence rate has been improved for accelerated version of CSCGD in \cite{thomdapu2021QoS} to $\cO\br{T^{-2/21}}$ which are the best-known result thus far. Differently, the theoretical analysis in this paper provides more general proof of convergence using supermartingale convergence argument with improved sample complexity results matching the results for the unconstrained setting in \cite{wang2017stochastic}. 


The rest of the paper is organized as follows. Sec. \ref{theSad} details the proposed algorithm and the relevant theoretical guarantees. Sec. \ref{simul} evaluates the performance of the proposed algorithm on fair classification and fair regression problems. Finally, Sec. \ref{concl} concludes the paper.

\noindent\textbf{Notation: } Small bold-faced letters represent column vectors, and bold-faced functional operator denotes vector function. For any vector $\x \in \Rn^n$, $\x^T$ represents its transpose and $\norm{\x} = \sqrt{\x^T\x}$ the Euclidean norm. For any matrix $\mathbf{A} = \Rn^{n\times n}$, its $\ell_2$ norm is denoted by $\norm{\mathbf{A}} = \max_{\norm{\x} = 1}\norm{\mathbf{A}\x}$. For any two sequences $\bc{a_t}$, $\bc{b_t}$, we denote $a_t = \cO\br{b_t}$ if there exists $c>0$, such that $\abs{a_t} \leq c\abs{b_t}$ for all $t$. Sets are denoted by capital letters in calligraphic font depending on the context. We denote $\Pi_{\cX}\bc{\y}$ as an operator that projects $\y$ onto the set $\cX$ as $\Pi_{\cX}\bc{\y} = \arg\min_{\x \in \cX} \norm{\y-\x}$. Similarly, $\po{\cdot}$ is denoted as projection onto the non-negative orthant. Finally, the operator $\nabla$ stands for either gradient, sub-gradient, or any directional gradient depending on the context and for a vector function $\g(\x)$, (sub-)gradient/directional derivative of it is denoted as a matrix $\nabla\g(\x)$, where $\bs{\nabla\g(\x) }_{ij} = \partial g_j(\x)/\partial x_i$. 

\section{Motivating Examples}\label{motex}
	Constrained compositional stochastic problems can be found at multiple applications of which, in this paper, we discuss some applications that require solving problems of the form \eqref{mainProb}. 
	
	\subsection{Risk-averse optimization}
		One interesting application of compositional optimization theory is a risk-averse optimization problem. Consider a portfolio selection problem in an index fund where, we are interested to find weights of some (say $N$) securities to invest in, such that the expected returns should be maximized and follow a particular index $\cI$. For that, we formulate the optimization problem as,
		\begin{align}\label{exPort}
		\underset{\x \in\cX}{\max}~ \sum_{i=1}^N\E{r(x_i;\xi_i)} \hspace*{3mm} \text{s.t. } ~\E{\bs{\br{ \sum_{i=1}^N\E{r(x_i;\xi_i)} - r(\mathcal{I};\zeta) }^p}_+}^{1/p} \leq \tau 
		\end{align}
		in which our objective is to maximize expected returns such that the deviation risk measure from the index $\cI$ is not more than the design specific value $\tau$. The function $r(\x_i;\xi_i)$ is a random reward value for the security $i$ and $\xi_i$ represents the randomness associated with that security $i$. The constraint function is deviation risk measure of the random payoff from the index $\cI$ and $\zeta$ is the random variable that is associated with the index $\cI$. In many practical cases, the distribution of random variables $\bc{\zeta}, \bc{\xi_1,..,\xi_N}$ is not known in advance and hence Monte-Carlo based algorithms are needed to solve problems such as \eqref{exPort}. Clearly, the constraint function is in compositional form of \eqref{mainProb} where the random variables  $\br{\xi_1,..,\xi_N}$ are associated with the inner expectation and $\zeta$ is associated with the outer expectation. For more examples on risk measure optimization problems, readers are referred to \cite{ruszczynski2006optimization}. 
	
	\subsection{Fairness aware classification} \label{faircl}
	We consider the problem of fair classification arising in automated decision making for critical processes such as hiring and promotions. The decisions made in the past may be biased against certain groups of candidates (eg. gender, ethnicity, or race). For example, consider an AI-assisted hiring decision making process. But the training data is biased with regards to the sensitive attribute gender because fewer females hired historically. Hence it is important to ensure that a classifier trained on the historical data be free of such biases. To this end, \cite{zafar2017fairness} proposed a fair classifier that sought to minimize the classification error while ensuring that the \emph{disparate impact} of the resulting decisions remains small. The constrained classification approach has been studied in  \cite{goh2016satisfying,kamishima2011fairness,menon2018cost,olfat2018spectral,woodworth2017learning,zafar2017fairness,zafar2019fairness} and a general framework for the same was proposed in  \cite{fairclf}. 
	
	Typically, fairness criteria are specified in terms of probabilities, which can be written as expectations of suitably defined indicator functions, but render the problem formulation non-smooth and non-convex. To this end, a relaxed formulation was obtained in \cite{fairclf} where a specific relaxation was used so as to avoid the compositional form. In this section we show that it is possible to obtain other similar formulations that better approximate the original problem. In particular, our ability to handle compositional constraint functions allows us to consider a larger variety of approximations, which are otherwise not possible.
	
	We start by discussing a slightly modified online version of the fairness-aware binary classification framework that is described in \cite{fairclf}. Let $\x$ denote a feature vector that does not contain any explicit information about the sensitive attribute $s\in S = \bc{s^{+},s^{-}}$, and let $y \in Y = \{0,1\}$ be the corresponding label. We assume that the norm of the feature vector $\norm{\x} \leq D_x$ is bounded. The binary classifier maps the feature vector into its corresponding label as $y = \text{sign}\br{\w^\T\x}$ where, $\w$ is the model parameter with bounded norm, i.e.,  $\norm{\w} \leq D_w$. For the purposes of classification, we consider the regularized logistic loss function
	\begin{align}
	L(\w)= \E{\log\br{1 + \exp\br{-y\cdot \w^{T}\x}}} + \frac{\mu}{2}\norm{\w}_2^2. \label{loss}
	\end{align} 
	where the expectation is with respect to the data point $(\x,y)$. The fairness is measured as the difference between the positive predictions in each group of sensitive attributes \cite{pedreschi2012study}. The risk difference function hence is given by \cite{fairclf}:
	\begin{align*}
	RD(\w) = \E{\indic_{\bc{\w^\T\x>0}}\lvert S = s^+} -  \E{\indic_{\bc{\w^\T\x>0}}\lvert S = s^- },
	\end{align*}
	and the fairness constraint is written as $\abs{RD(\w)} \leq \tau$, where $\tau$ the design specific parameter. Note here that indicator function $\indic_{z>0}$ is 1 if the condition $z > 0$ is true and zero otherwise. Further simplification yields $RD(\w)$
	\begin{align}
	&\!=\! \E{\frac{\P{S=s^{+}|\x}}{\P{S=s^{+}}}\indic_{\w^\T\x >0}\! +\! \frac{\P{S=s^{-}|\x}}{\P{S=s^{-}}}\indic_{\w^\T\x \leq 0} - 1} \label{indi}
	\end{align}
	which involves the expectation term in the denominator thus making it as an expression of non-linear functions of expectation. Since we are interested in designing online classifier, the simulation environment is considered to be blackbox where, the data samples $\bc{\x_i, s_i, y_i}_i$ are unknown but can be accessed in sequential manner. The random variables $\br{S|\x, S}$ are associated with inner expectation and $\x$ is associated with the outer expectation. Hence, solving the online classification problem with fairness constraints can be viewed as a spacial case of constrained compositional stochastic problem \eqref{mainProb} (constraint is compositional):
	\begin{align}
		\min L(\w) \hspace{3mm} \text{s.t } \abs{RD(\w)} \leq \tau.
	\end{align}
	Note that the expression in \eqref{indi} contains indicator functions making the constraint non-convex. replacing indicator functions with appropriate surrogates is common practice in machine learning literature \cite{fairclf}. By making use of the flexibility that is offered by the more general structure \eqref{mainProb}, we develop three intuitive ways of expressing the constraint function. We will give more details on such formulations in numerical section (see sec. \ref{fcl}).

	\subsection{Fair SpAM}\label{spamsec}
	Another interesting motivating example is fairness aware regression problem. Bias within the historical data may also result in unfair regression models, where the decision variables are continuous \cite{komiyama2018nonconvex}. Continuing our earlier example of assessing the performance of employees, pertinent regression problems include learning models that predict their efficiency in performing certain tasks or estimate their salaries. As in the case of classification, regression models may again lead to biased predictions, if such biases are also present in the historical data. In this section,  we study a related fair regression problem, where non-linear functions of expectations also arise. Fairness is ensured by imposing constraints on the model.

	A family of fairness metrics useful for regression is studied in \cite{berk2017convex}, where the constraints arising from the fairness requirements appear as penalties in the objective function. In this section, we consider Sparse Additive Models (SpAM) \cite{spam} which is a class of models used for sparse high dimensional non-parametric regression. To enforce the fairness, we utilize a  constraint that restricts correlation between the prediction and sensitive attribute, similar to the ideas in \cite{zafar2019fairness,zink2019fair}. 
	
	SpAM models are used when the feature space of the data is large but contains several irrelevant features, thus necessitating a sparse prediction function. Given a data point $\br{\x_i,y_i}$, such that $\x_i \in \Rn^d$, SpAM models postulate that $	y_i = \sum_{j=1}^d h_j(x_{i j}) + \epsilon_i$ where, each $h_j :\bR \rightarrow \bR$ is a scalar feature function, $x_{ij}$ is the $j$-th entry of $\x_i$, and $\epsilon_i$ is the noise. 
	The idea is to obtain a sparse set of features, e.g., by making several of the feature functions zero. In the special case when the feature functions $h_j$ are linear, the estimation of feature functions is readily achieved by making use of the $\ell_1$ norm penalty (LASSO estimator \cite{tibshirani1996regression}). However In the more general non-linear case, the $\ell_1$ norm penalty can no longer guarantee the sparsity. Hence,  similar to the SpAM application in \cite{wang2017accelerating}, we consider the following objective function where $\ell_2$ norm penalty function is being used.
	\begin{align*}
	\min_{h_j \in \cH_j , j=1,\ldots,d} \!\E{\br{\!y\! -\! \sum_{j=1}^d\! h_j(x_j)}^2}\! +\! \mu\sum_{j=1}^d\! \sqrt{\!\E{\bs{h_j^2(x_j)}}}.
	\end{align*}
	where $\cH_j$ is the considered class of scalar functions. Since the problem unconstrained with the objective being a non-linear function of expectation, it can be solved using the SCGD algorithm \cite{wang2017stochastic}. Consider the case each data point $(\x,y)$ is associated with a sensitive attribute $s \in \Rn$. Then the fairness in the regression model can be enforced by ensuring that the covariance between the sensitive attribute $s$ and the prediction $\sum_{j=1}^d h_j(x_j)$ is small in magnitude. Hence to enforce fairness, the following constraint is added.  
	\begin{align}\label{fairspam}
	-\tau \leq \E{s \sum_{j=1}^d h_j(x_j)} -\E{s}\E{\sum_{j=1}^d h_j(x_j)} \leq \tau,  
	\end{align}
	where $\tau$ is the design specific parameter. The constraint in \eqref{fairspam} is stochastic in nature.  Our goal here is to learn the SpAM predictor $\bc{\hat{h_j}}_j$ utilizing the random data samples $\bc{\x_i,y_i}_i$ that are revealed sequentially. Hence the online fairness aware SpAM problem can also be viewed as a spacial case of constrained compositional stochastic problem \eqref{mainProb} (Objective is compositional). More discussion on numerical details of SpAM problem is provided in Sec. \ref{numspam}.
		
	Other examples related to designing optimal queuing systems (see \cite{thomdapu2019optimal}, \cite{thomdapu2021QoS}) can also be formulated as a spacial case of \eqref{mainProb}. Note that all of these examples has the compositional structure in at-least one of objective and constraint functions and contains random variables that are associated with both inner functions ($\xi, \phi$) and outer functions ($\zeta, \psi$). However, to ensure more generality in theory, we consider problem \eqref{mainProb} that subsumes all of those diverse formulations. We remark that, there exists no first order online algorithm (other than in \cite{thomdapu2019optimal,thomdapu2021QoS}) in the literature to solve these kind of problems.

\section{Compositional Stochastic Saddle Point Method}\label{theSad}
This section puts forth the proposed Compositional Stochastic Saddle Point Algorithm (CSSPA) and analyzes its performance. We begin by stating the assumptions and then establish some preliminary results (Lemmas \ref{lempre}-\ref{lagdual}), which are subsequently used to prove the main results (Theorems \ref{althm}-\ref{bcthm}). For the sake brevity, let $\gb(\x) := \EE_{\xi}[\g(\x;\xi)]$ and $\hb(\x) := \EE_{\phi}[\h(\x;\phi)]$  for $\x\in\cX$. Henceforth, we drop the subscripts from the expectations if they can be inferred from the context. 


\subsection{Algorithm}
We motivate the proposed algorithm by considering the dual problem first. Associating dual variables $\{\lambda_j\}_{j=1}^J$ with the constraints in \eqref{mainProb}, the Lagrangian is given by 
\begin{align} \label{claslag}
\Lh(\x,\lam) = F(\x) + \sum_{j=1}^J \lambda_jL_j(\x).
\end{align}
Collecting the dual variables in the vector $\lam \in \Rn^{J}$, the dual problem is given by
\begin{align}\label{dualProb}
\vls = \arg\max_{\lam\geq 0}\mathcal{D}\br{\lam},
\end{align}
where the dual function $\mathcal{D}\br{\lam}$ is given by
\begin{align}
\mathcal{D}\br{\lam} = \min_{\x \in \cX}~\Lh(\x,\lam).
\end{align}
If \eqref{mainProb} is convex and satisfies the Slater's condition (see Assumption \ref{aslater}), the duality gap is zero, i.e.,  $F(\xs)-\mathcal{D}\br{\vls} = 0$, and the optimal $\lam^\star$ is bounded, i.e.,  $\norm{\vls} \leq B_{\lamb} < \infty$. 

To solve \eqref{mainProb}, we will consider the Arrow-Hurwicz saddle point algorithm, which utilizes an augmented Lagrangian of the form:
\begin{align}
\cL\br{\x,\lam,\alpha,\delta} = F(\x) + \sum_{j=1}^J\lambda_jL_j(\x)  - \frac{\alpha\delta}{2}\norm{\lam}^2. \label{lag}
\end{align}
Here, $\alpha$ and $\delta$ are positive parameters that result in a slightly different form of the update than classical saddle point algorithms. We will establish that the use of the augmented Lagrangian does not hurt the convergence rate, provided $\alpha$ and $\delta$ are sufficiently small. In contrast, the augmentation allows us to forgo any explicit assumptions on the boundedness of the dual variable $\lam$. Observe further that different from $\Lh$, the augmented Lagrangian $\cL$ is $\alpha\delta$-strongly concave in $\lam$.

The nested expectations in \eqref{mainProb} will be handled by using approximations to the augmented Lagrangian gradient. In particular, we pursue a quasi-gradient saddle point algorithm that entails carrying out parallel primal and dual updates of the form:
\begin{align}
\x_{t+1} &= \px{\x_t - \alpha_{t} \hat{\nabla}_\x\cL\br{\x_t,\lam_t,\alpha_t,\delta_t}} \label{prim}\\
\lam_{t+1} &= \po{\lam_t + \alpha_{t} \hat{\nabla}_{\lam} \cL\br{\x_t,\lam_t,\alpha_{t},\delta_t}} \label{dual}
\end{align}
where $\hat{\nabla}_\x\cL\br{\x_t,\lam_t,\alpha_t,\delta_t}$ and  $\hat{\nabla}_{\lam}\cL\br{\x_t,\lam_t,\alpha_t,\delta_t}$ are approximations to the true gradients $\nabla_\x\cL\br{\x_t,\lam_t,\alpha_t,\delta_t}$ and  $\nabla_{\lam}\cL\br{\x_t,\lam_t,\alpha_t,\delta_t}$, respectively. Since the Lagrangian is a non-linear function of expectations, classical unbiased stochastic approximations to these gradients cannot be obtained by simply dropping these expectations. Instead, as in \cite{wang2017stochastic}, these approximations are constructed by maintaining the auxiliary variables
\begin{align}
\y_{t+1} &= \br{1 - \beta_t}\y_t + \beta_t \g\br{\x_t;\xi_t} \label{auxu11}\\
\w_{t+1} &= \br{1 - \beta_t}\w_t + \beta_t \h\br{\x_t;\phi_t}\label{auxu21}
\end{align}
for tracking $\gb(\x_t)$ and $\hb(\x_t)$, respectively. These auxiliary variables are subsequently used to construct the gradient approximations
\begin{align}
\hat{\nabla}_\x\cL\br{\x_t,\lam_t,\alpha_t,\delta_t} &= \nabla \g\br{\x_t;\xi_t}\nabla f\br{\y_{t+1};\zeta_t} + \sum_{i=1}^J\lamb_{j,t}\nabla \h\br{\x_t;\phi_t}\nabla \ell_j\br{\w_{t+1};\psi_t} \\
\hat{\nabla}_{\lam} \cL\br{\x_t,\lam_t,\alpha_{t},\delta_t} &= \ell_j\br{\w_{t+1};\psi_t} - \alpha_t\delta_t.
\end{align}

The first result, detailed in Theorem \ref{althm}, considers the updates in \eqref{prim}-\eqref{dual}, and establishes the almost sure convergence of the sequence $(\x_t,\lam_t)$ to the optimum $(\x^\star,\lam^\star)$.

It may be remarked however that in practice, one may only run the algorithm for a fixed number of iterations, and hence the non-asymptotic performance of the proposed algorithm is of interest. Such rate results are generally obtained by adopting a slightly different analytical approach, and by bounding metrics such as the optimality gap for finite $T \geq 1$. A complication that arises in the context of constrained optimization is that even though $\x_t \rightarrow \x^\star$ almost surely, it is not necessary that $\x_t$ be feasible for any finite $t$. In fact, for constrained stochastic optimization, the expected optimality gap and the expected constraint violation, both are known to decay as $\O(T^{-\frac{1}{2}})$ after $T$ iterations  \cite{madavan2019subgradient, yu2017online}. Likewise, it will shown here that updates \eqref{prim}-\eqref{dual} applied to solve \eqref{mainProb} will result in the optimality gap and constraint violation decaying as $\O(T^{-\frac{1}{4}})$. In other words, unlike in unconstrained optimization, stopping the algorithm too early incurs the double penalty of suboptimality as well as infeasibility. 

It is however possible to achieve zero constraint violation for all $t \geq 1$, while still achieving the same bound on the optimality gap, by constructing a  non-asymptotic version of the proposed algorithm following the approach in  \cite{akhtar2020conservative}. To this end, we consider the intermediate problem:
	\begin{align}\label{sur}
	\x^{\theta} = \arg\min_{\x\in\cX} F(\x) \tag{$\mathcal{P}_{\theta}$}\hspace{3mm} \text{s.t. } L_j(\x) + \theta \leq 0\text{ } \forall 1\leq j \leq J \nonumber
	\end{align} 
	where we introduce a non-negative tightening parameter $\theta$. The iterations \eqref{prim}-\eqref{dual} are subsequently applied to solve \eqref{sur} instead of \eqref{mainProb}. Letting $\hat{\x}$ denote the output of the proposed algorithm after $T$ iterations, we will establish in Theorem 2 that the optimality gap $\EE[F(\hat{\x})] - F(\x^\theta)$ as well as the constraint violation $\max_j \EE[L_j(\hat{\x})] + \theta$ decay as $\O(T^{-\frac{1}{4}})$ after $T$ iterations. We will further establish that  $F(\x^\theta)-F(\x^\star) \approx \O(\theta)$, which would imply that by choosing $\theta = \O(T^{-\frac{1}{4}})$, we can ensure that the constraint violation $\max_j \EE[L_j(\hat{\x})] $ is zero, while the optimality gap is $\EE[F(\hat{\x})] - F(\x^\star)$ is still $\O(T^{-\frac{1}{4}})$. 

In summary, the proposed algorithm entails the application of quasi-gradient saddle point updates 
\eqref{prim}-\eqref{dual} to \eqref{sur}. The complete algorithm is summarized in Algorithm \ref{alg1}. The algorithm starts with an arbitrary initialization and at each iteration, requires $m+1+d+J$ calls to the stochastic gradient oracle, corresponding to $\nabla \g\br{\x_t;\xi_t}$, $\nabla f\br{\y_{t+1};\zeta_t}$, $\nabla \h\br{\x_t;\phi_t}$, $\{\nabla \ell_j\br{\w_{t+1};\psi_t}\}_{j=1}^J$, and $m+d$ calls to the function evaluation oracle, corresponding to $\g(\x_t,\xi_t)$ and $\h(\x_t,\phi_t)$. Since the per-iteration complexity of the algorithm is fixed, all bounds will be presented directly in terms of the number of iterations instead of the number of oracle calls.       

\begin{algorithm}
	\caption{Compositional Stochastic Saddle Point Algorithm}
	\begin{algorithmic}[1]
		\label{alg1}
		
		\REQUIRE $\x_1,\y_1,\w_1,\lam_1, \text{ step sizes } \bc{\alpha_t, \beta_t} \in [0,1)$, parameters $\br{\bc{\delta_t},\theta}\geq 0$
		\ENSURE $\bc{\x_t, \y_t, \w_t, \lam_t}$
		\FOR{$t = 1 \text{ to } T$}
		\STATE Sample random variables $\xi_t,\zeta_t,\phi_t,\psi_t$.
		\STATE Update the auxiliary variables 
		\begin{align}
		\y_{t+1} &= \br{1 - \beta_t}\y_t + \beta_t \g\br{\x_t;\xi_t} \label{auxu1}\\
		\w_{t+1} &= \br{1 - \beta_t}\w_t + \beta_t \h\br{\x_t;\phi_t}\label{auxu2}
		\end{align}
		\STATE Update the primal variable  \begin{align}\label{primu}
		&\x_{t+1} = \Pi_{\cX}\Bigg\{\x_t - \alpha_t\Bigg(\nabla \g\br{\x_t;\xi_t}\nabla f\br{\y_{t+1};\zeta_t} \nonumber\\
		&\hspace{5mm}+ \sum_{i=1}^J\lamb_{j,t}\br{\nabla \h\br{\x_t;\phi_t}\nabla \ell_j\br{\w_{t+1};\psi_t}}\Bigg)\Bigg\}
		\end{align}
		\STATE Update the dual variable \begin{align} \label{duau}
		&\lamb_{j,t+1} = \Big[\lamb_{j,t}\br{1-\alpha_t^2\delta_t}\nonumber\\
		&\hspace{5mm}+ \alpha_t\br{\ell_j\br{\w_{t+1};\psi_t} + \theta}\Big]_{+} \hspace{3mm} \forall j=1,...,J
		\end{align}
		\ENDFOR
		\STATE \textbf{Output:} 
		$\hat{\x} = \br{1/\sum_{t=1}^T \alpha_{t}}\sum_{t=1}^T\alpha_{t}\x_{t}$
	\end{algorithmic}
\end{algorithm}

\subsection{Assumptions}\label{assump}
Before proceeding to the analysis, we state the necessary assumptions on the optimization problem \eqref{mainProb}. 
\begin{assumption}\label{aiid}
	The random samples $\bc{\xi_t, \zeta_t, \phi_t, \psi_t}$ are drawn in an independent identically distributed fashion for each $t$. Further, for each $t$, the random variables $\xi_t$ and $\zeta_t $ are independent, and likewise, $\phi_t$ and  $\psi_t$ are independent.  
\end{assumption}

\begin{assumption}\label{aslater}
	The problem \eqref{mainProb} is a strictly feasible convex optimization problem, i.e., there exists $\xt$ such that $\max_j L_j(\xt) + \sigma_0 \leq 0$ for some constant $\sigma_0 > 0$. The set $\cX$ is  proper, closed, and compact, i.e., 
	\begin{align*}
		\sup_{\x_1,\x_2 \in \cX} \norm{ \x_1- \x_2}^2 \leq D_x < \infty.
	\end{align*}
\end{assumption}

\begin{assumption}\label{adiff}
	The outer functions $f$ and $\lj$ are continuously differentiable and the inner functions $\g$ and $\h$ are continuous. Consequently, the (sub-)gradients of the objective and constraint functions are well-defined, with
	\begin{align}\label{as1}
	\E{\nabla \g (\x;\xi)\nabla f\br{\gb(\x);\zeta}} &= \nabla\gb(\x)\E{\nabla f\br{\gb(\x);\zeta}} &\in \nabla F(\x) \\
	\E{\nabla \h (\x;\phi)\nabla \ell_j\br{\hb(\x);\psi}} &= \nabla\hb(\x)\E{\nabla \ell_j\br{\gb(\x);\psi}} &\in \nabla L_j(\x) 
	\end{align}
	for all $1\leq j \leq J$.
\end{assumption}

\begin{assumption}\label{ainner}
	The functions $\g$ and $\h$ are Lipschitz continuous in expectation and have bounded variance, i.e., for all $\x\in\cX$, it holds that
	\begin{align}
	\E{\norm{\nabla \g(\x;\xi)}^2} \leq &C_g  &\E{\norm{\g(\x;\xi) - \gb(\x)}^2} \leq &V_g \\
	\E{\norm{\nabla \h(\x;\phi)}^2} \leq &C_h  &\E{\norm{\h(\x;\phi) - \hb(\x)}^2} \leq &V_h 
	\end{align}
\end{assumption}

\begin{assumption}\label{aoutersm}
	The functions $f$, and $\lj$ are smooth with probability one, i.e., for all $\y_1$, $\y_2 \in \Rn^m$, and $\w_1$, $\w_2 \in \Rn^d$, it holds that
	\begin{align}
	\abs{\nabla f\br{\y_1;\zeta} - \nabla f\br{\y_2;\zeta}} &\leq L_f \norm{\y_1 - \y_2},&\abs{\nabla \ell_j \br{\w_1;\psi} - \nabla \ell_j \br{\w_2;\psi}} &\leq L_{\ell} \norm{\w_1 - \w_2}
	\end{align}	
	with probability one for all $1 \leq j \leq J$. 
\end{assumption}

\begin{assumption}\label{aoutergra}
	The stochastic gradients of functions $f$ and $\lj$ have bounded second order moments, i.e., for all $\y\in\Rn^m$ and $\w\in\Rn^d$, we have 
	\begin{align}
	\E{\norm{\nabla f\br{\y;\zeta}}^2 } &\leq C_f, & \E{\norm{\nabla \ell_j\br{\w;\psi}}^2} &\leq C_{\ell} 
	\end{align}	
\end{assumption}

\begin{assumption}\label{aconsb}
	The constraint functions $\lj$ have bounded second moments, i.e., for all $\x \in \cX$, it holds that
	\begin{align}
	\E{\br{\ell_j\br{\hb(\x;\phi);\psi}}^2} &\leq B_{\ell} 
	\end{align}
	for all $1\leq j \leq J$.
\end{assumption}

The stated assumptions are standard and similar to those made in \cite{wang2017stochastic,wang2017accelerating, thomdapu2019optimal,thomdapu2021QoS} and hold for many problems of interest. First note that for $\theta < \sigma_0/2$, the surrogate problem  \eqref{sur} is also strictly feasible, implying that its duality gap is zero and its dual optimal $\vlth$ is bounded as $\norm{\vlth} \leq B_{\lamb} < \infty$. From Assumptions \ref{adiff} and \ref{ainner}, we can see that the inner functions $\g$ and $\h$ are not required to be smooth or convex, as long as the composite functions $F$ and $\{L_j\}_{j=1}^J$ are convex.The assumption in \eqref{aoutersm} is strict in the sense that the outer functions must have Lipschitz continuous gradients for almost any realization of $\zeta$ and $\psi$. In general, many practical applications do not have randomness associated with the nonlinear outer functions (see \cite{wang2017stochastic,wang2017accelerating,thomdapu2019optimal,thomdapu2021QoS}) and satisfy the assumption in \eqref{aoutersm}. In other hand, the assumption in \ref{aoutergra} is pretty standard in stochastic approximation literature, and generally hold for many applications. Finally, Assumption \ref{aconsb} requires that the zero-th order value of constraint functions $\lj$ is bounded.

In subsequent sections, we discuss almost sure convergence and sample complexity analysis
for objective error and constraint violations. Complete analysis with all the intermediate
results has been provided in Appendix \ref{allproofs}.

\subsection{Almost Sure Convergence}
In this subsection, we use the super martingale convergence theorem to establish that the iterates generated by Algorithm \ref{alg1} converge almost surely. As remarked earlier, the asymptotic version of the algorithm directly targets \eqref{mainProb} and uses $\theta = 0$ in Algorithm \ref{alg1}. In contrast, the non-asymptotic version will require careful tuning of $\theta$ in order to achieve a zero constraint violation at a given iteration. The proof of theorem \ref{althm} is provided in Appendix \ref{pfalthm}. 
\begin{thm}[\bf{Almost Sure Convergence}]\label{althm}
	Suppose assumptions in Sec.\ref{assump} hold. Let $\bc{\x_t,\lam_{t}}$ be the sequence generated by Algorithm \ref{alg1} (with $\theta$ = 0) with the step sizes $\alpha_{t}, \beta_{t}$ and the parameter $\delta_t$ are selected such that 
		\begin{align}\label{stcals}
		\sum_{t=1}^\infty\! \alpha_{t}\! =\! \infty,\hspace{1mm}\sum_{t=1}^\infty\! \beta_{t}\! =\! \infty,\hspace{1mm}	\sum_{t=1}^\infty\! \br{\!\alpha_{t}^2\! +\! \beta_{t}^2\! +\! \frac{\alpha_{t}^2}{\beta_{t}}\! +\! \frac{\alpha_{t}^2}{\beta_{t+1}} }\! <\! \infty
		\end{align} 
		and $\delta_t = 2K\br{1+\frac{1}{\beta_{t}} + \frac{1}{\beta_{t+1}}}$, where $\alpha_{t} = (1/4K) t^{-a}, \beta_{t} = t^{-b}$,
		and
		\begin{align*} 
			K = \max\!\bc{\!4, JL_{\ell}^2C_hD_x, 2JC_hC_{\ell}, 8JC_hC_{\ell}\!\br{1\!+\!4JC_l}\br{C_g\!+\!C_h}\!}.
		\end{align*} 
		Then the sequence $(\x_t,\lam_{t})$ converges almost surely to $(\xs,\vls)$.
\end{thm}

The convergence of the iterates also implies that the optimality gap and the constraint violation go to zero with probability one.

\begin{table*}
	\centering
	\def\arraystretch{1.6}
	\setlength{\tabcolsep}{5pt}
	\begin{tabular}{|c|c|c|c|c|}
		\hline
		\textbf{Step sizes }                                                           & \textbf{Choice}                              & $\E{F(\hat{\x})}\! -\! F^\star$            & $\max_j\! \E{L_j(\hat{\x})}$ & $\mathbf{\theta}$         \\ \hline
		\multirow{2}{*}{$\alpha_t = \alpha_0t^{-a}, \beta_t = \beta_0t^{-b}$} & \multirow{4}{*}{$a = 3/4, b = 1/2$} & \multirow{2}{*}{$\cO\br{T^{-1/4}\log{T}}$} & $\cO\br{T^{-1/4}\log{T}}$    & 0                         \\ \cline{4-5} 
		&                                     &                                            & 0                            & $\cO\br{T^{-1/4}\log{T}}$ \\ \cline{1-1} \cline{3-5} 
		\multirow{2}{*}{$\alpha_t = \alpha_0T^{-a}, \beta_t = \beta_0T^{-b}$} &                                     & \multirow{2}{*}{$\cO\br{T^{-1/4}}$}        & $\cO\br{T^{-1/4}}$           & 0                         \\ \cline{4-5} 
		&                                     &                                            & 0                            & $\cO\br{T^{-1/4}}$        \\ \hline
	\end{tabular}
	\caption{Error Bounds for Convex Case} \label{tab:title}
\end{table*}

\subsection{Rate of Convergence}
After concluding that the iterates in Algorithm \ref{alg1} converge to a limit point with probability 1, we now analyze the sample complexity of objective error and constraint violations in terms of step sizes $\alpha_{t}$, $\beta_{t}$.

\begin{thm}
	[\bf{Rate of Convergence}]\label{bcthm}
	Suppose assumptions in Sec.\ref{assump} hold. Let $\bc{\x_t,\lam_{t}}$ be the sequence generated by Algorithm \ref{alg1} with the step sizes $\alpha_{t}, \beta_{t}$ and the parameter $\delta_t$ are selected such that 
	\begin{itemize}
		\item $\alpha_t = (1/7K)T^{-a}$, $\beta_t = T^{-b}$, $\forall t$ or
		\item $\alpha_t = (1/7K)t^{-a}$, $\beta_t = t^{-b}$ 
	\end{itemize}
	and $\delta_t = 4K\br{1+\frac{1}{\beta_{t}} + \frac{1}{\beta_{t+1}}}$, where
	\begin{align}\label{chK}
		K = 2J\max\Big\{2,C_hC_{\ell}((1+4JC_{\ell})C_h+C_g+C_h), C_{\ell} + L_{\ell}^2C_hD_x, C_hC_{\ell}\Big\}
	\end{align}
	and $(a,b) \in (0,1)$, $b\geq a$, $a\leq 2b$. Then at $\hat{\x}$ (denoted in Algorithm \ref{alg1}), the optimality gap and constraint violations are upper bounded as
	\begin{align*}
	\E{F(\hat{\x})} - F(\xs) &\leq \omega + \theta\frac{2\sqrt{C_fC_g}D_x}{\sigma_0} \hspace{5mm} \max_{j = 1,...,J} \E{L_j(\hat{\x})} \leq \omega\br{2+ \frac{8C_fC_gD_x^2}{\sigma_0^2}} -\theta
	\end{align*}
	where $\omega = \cO\br{T^{a-1}} \sum_{t=1}^T\br{\frac{1}{T} + \alpha_{t}^2 + \beta_t^2+ \frac{\alpha_{t}^2}{\beta_{t}} + \frac{\alpha_{t}^2}{\beta_{t+1}}}.$
\end{thm}

While similar results exist for classical (non-compositional) saddle point algorithms, the resulting rates are suboptimal. For instance, the approaches in \cite{koppel2015saddle,bedi2019asynchronous} entail choosing a specific value of $\lt$ and provide the result for any feasible point $\xt$ that is not necessarily the saddle point $\br{\x^\theta,\lam^{\theta}}$ of \eqref{sur}. Such an approach results in the optimality gap decaying as $\O\br{T^{-1/2}}$ while constraint violations decaying as  $\O\br{T^{-1/4}}$. In the current context, we present a new approach where we choose $\br{\xt,\lt}$ as a function of the saddle point $\br{\x^\theta,\lam^{\theta}}$ in order to achieve better rates. The proof is provided in Appendix \ref{pfbcthm}.
\begin{figure}
	\centering
	\includegraphics[scale=0.5]{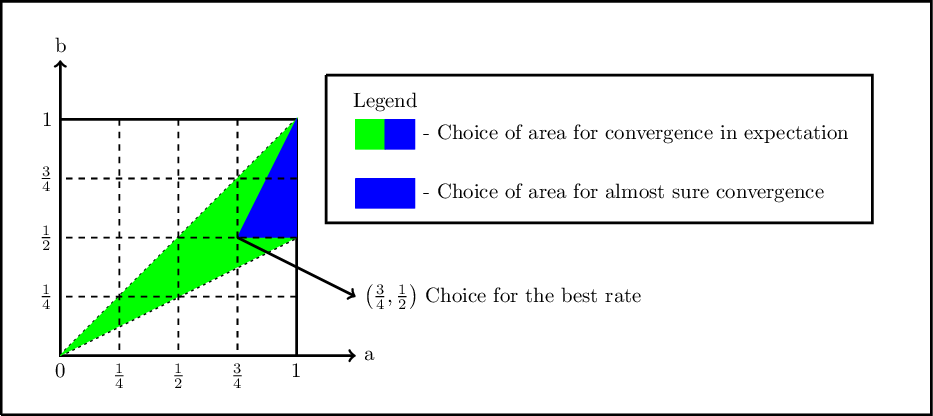}
	\caption{Selection of $a,b$}
	\label{ratech}
\end{figure}

	We remark that, in theorems \ref{althm}, \ref{bcthm}, since $K$ is the function of $J, L_{\ell}, C_g, C_h, C_{\ell}, D_x$, knowing these constants is enough to set the parameters $\alpha_{t}$, $\beta_{t}$ and $\delta_t$. However, in practice, we may not know these constants in case of which we can tune the parameters. The parameter $K$ is simply the maximum value of such constants and hence the step sizes can easily be tuned as we know the dependence of $K$ on $\alpha_t$, $\beta_t$ and $\delta _t$. The results in Theorems \ref{althm} and \ref{bcthm} depend on the constants $a$ and $b$. Fig. \ref{ratech} summarizes	the possible valid choices of $a$ and $b$. For instance, the choice of $a$ and $b$ within the blue shaded region allows for almost sure convergence. On the other hand, the rate results in Theorem \ref{bcthm} apply for all choices of $a$ and $b$ within the green and blue shaded region. From Table \ref{tab:title}, we remark that, with constant step sizes, the achievable rate of convergence for optimality gap with constraints is $\cO\br{T^{-1/4}}$ and same as the rate for unconstrained problem in \cite{wang2017stochastic}. However, with diminishing step sizes, the rate is slightly worse $\cO\br{T^{-1/4}\log T}$ compared to algorithm SCGD  for unconstrained problems. Notice that the particular selection of step sizes in Table \ref{tab:title} satisfies the condition on step sizes as mentioned in the statement of Lemma \ref{lagl}. Further, we argue that by carefully tuning the parameter $\theta$ as $\theta = \omega\br{2+ \frac{8C_fC_gD_x^2}{\sigma_0^2}}$, constraint violation can be made zero. Interestingly, the results in Table \ref{tab:title} state that for sufficiently large $T$, it is possible to ensure that $\theta \leq \sigma_0/2$. Also notice that, when we are interested in asymptotic result which has been proved in Theorem \ref{althm}, we assume that the updates in Algorithm \ref{alg1} are run infinite number of times and it is easy to verify that, $\lim_{T\rightarrow\infty} \theta =0$. Hence, we remark that the choice of $\theta =0$ is consistent with the argument of Theorem \ref{althm}. Note that as in assumption \eqref{aslater}, although the problem in \eqref{mainProb} satisfies the slater condition, it is imperative to know the value of $\sigma_0$ to tune $\theta$. there are ways to determine $\sigma_0$. One is, looking at the actual structure of the constraint functions which is very problem specific. Other is, to solve an auxiliary problem where we minimize $\max_j\bc{L_j(\x)}$ and the solution is highly likely to be some $\x^\sigma$ such that $\sigma_0$ is positive. This can be done by keeping a very small part of the data (5-10\%) for the tuning. However, If the knowledge of $\sigma_0$ is unavailable, we choose $\theta =0$.

	\section{Numerical Results} \label{simul}
	This section provides the numerical results that demonstrate the efficacy of the proposed algorithm in a number of settings. We discuss the problems of classification and sparse additive models with enforced fairness constraints.  
	
	\subsection{Fair classifier}\label{fcl}
		In this section, we conduct experiments on fair-classification problem that has been discussed in Sec.\ref{faircl}. Since the indicator function within the definition of risk difference function \eqref{indi} renders the constraint non-convex in the optimization variable $\w$, appropriate relaxations must be used. We discuss three such approximations and evaluate their performance on a real dataset to demonstrate the flexibility of the proposed framework. Throughout, the classification loss remains the same, while only some of the indicator functions are relaxed. The risk difference function in \eqref{indi} can be further simplified as
		\begin{align*}
			RD(\w) = \frac{\E{\indic_{S=s^{+}}\indic_{\w^\T\x >0}}}{\E{\indic_{S=s^{+}}}} + \frac{\E{\indic_{S=s^{-}}\indic_{\w^\T\x \leq 0}}}{\E{\indic_{S=s^{-}}}} - 1
		\end{align*} 
	
	\subsubsection{\textbf{Approximation 1}}\label{ap1} We write the constraint $\abs{RD(f)} \leq \tau $ as two separate constraints  $RD(f) \leq \tau$ and $-RD(f)\leq \tau $. We replace the indicator functions $\indic_{z>0}$ and $-\indic_{z>0}$ with convex relaxations $\max(0,1+z)$ and $-\min(1,z)$, respectively. This particular approximation was also proposed in \cite{fairclf}. These approximations allow us to re-write the constraint as
		\begin{align}
		\frac{\E{\indic_{S=s^{+}}\max\br{0,1+\w^T\x}}}{\E{\indic_{S=s^{+}}}}  + \frac{\E{\indic_{S=s^{-}}\max\br{0,1-\w^T\x}}}{\E{\indic_{S=s^{-}}}} - 1 &\leq c_1\tau \label{kappa_delta1}\\
		1-\frac{\E{\indic_{S=s^{+}}\min\br{1,\w^T\x}}}{\E{\indic_{S=s^{+}}}} - \frac{\E{\indic_{S=s^{-}}\min\br{1,-\w^T\x}}}{\E{\indic_{S=s^{-}}}} &\leq c_1\tau, \label{kappa_delta2}   
		\end{align}
		where $c_1$ is a tuning parameter which is tuned for the training data such that the feasibility of original constraint $\abs{RD(\w)} \leq \tau$ can be achieved by satisfying \eqref{kappa_delta1}-\eqref{kappa_delta2}. Note that although expressions in \eqref{kappa_delta1}-\eqref{kappa_delta2} are still non-linear functions of expectations, the corresponding components of the inner functions do not depend on $\w$. Therefore, we can separately track $\E{\indic_{S=s^{+}}}$ as well as $\E{\indic_{S=s^{-}}}$, and substitute. Expressions in \eqref{kappa_delta1}-\eqref{kappa_delta2} can be written in the form required in \eqref{mainProb} as
		\begin{align*}
		\vl(\vz) &= \bs{\frac{z_2}{z_1} + \frac{z_5}{z_4} - 1 - \tau, 1-\frac{z_3}{z_1}  -\frac{z_6}{z_4}  -c_1\tau}
		\end{align*}
		\begin{align*}
		\h(\w;\mathbf{x},y,S) &= \bigg[\indic_{S=s^{+}},\indic_{S=s^{+}}\max\br{0,1+\w^T\x},\indic_{S=s^{+}}\min\br{1,\w^T\x},\indic_{S=s^{-}},\\
		&\hspace{5mm}\indic_{S=s^{-}}\max\br{0,1\!-\!\w^T\x},\indic_{S=s^{-}}\min\br{1,-\w^T\x}\bigg]
		\end{align*}
		The denominators $z_1$ and $z_4$ in the arguments of outer function $\vl(.)$ lie in $\bs{0,1}$ which blow up the function value when very close to 0. We use a Huber-like approximation to address the issue which ensures that functions satisfy assumptions in Sec. \ref{assump}. For example the function $z_2/z_1$, can be replaced by a Huber approximation $H(.)$ which has the following form
		\begin{equation}
		H(z_1,z_2) = \begin{cases}
		\frac{z_2}{z_1} & \text{ if $z_1 > \epsilon$}\\
		\frac{z_2}{\epsilon}\br{2 - \frac{z_1}{\epsilon}} & \text{ if $z_1 \leq \epsilon$}\\
		\end{cases}
		\end{equation}
		Similarly, other functions can also be replaced. Hence $\norm{\w}$ and $\norm{\x}$ remain bounded and the constraint function is smooth as required by the assumptions.
	
	\subsubsection{\textbf{Approximation 2}} We re-write the constraint $\abs{RD\br{\w}} \leq \tau$ as $\br{RD\br{\w}}^2 \leq \tau^2$ and use the convex approximation $\max(0,1+z)$ to indicator function $\mathbb{I}_{z>0}$. To this end, we re-write the constraint as
		\begin{align}\label{appr2}
		&\bigg(	\frac{\E{\indic_{S=s^{+}}\max\br{0,1+\w^T\x}}}{\E{\indic_{S=s^{+}}}} + \frac{\E{\indic_{S=s^{-}}\max\br{0,1-\w^T\x}}}{\E{\indic_{S=s^{-}}}} - 1 \bigg)^2 \leq c_2^2\tau^2 
		\end{align}
		where $c_2$ is a tuning parameter. The constraint in \eqref{appr2} is expressed as non linear function of expectations and hence the proposed method is applicable here. The function in \eqref{appr2} can be written as in the form \eqref{mainProb} as
		\begin{align*}
		\vl(\vz) = \br{\frac{z_2}{z_1} + \frac{z_4}{z_3} - 1}^2 -c_2^2\tau^2
		\end{align*}
		\begin{align*}
		&\h(\w; \mathbf{x},y,S) = \bigg[\indic_{S=s^{+}},\indic_{S=s^{+}}\max\br{0,1+\w^T\x},\indic_{S=s^{-}},\indic_{S=s^{-}}\max\br{0,1-\w^T\x}\bigg]. 
		\end{align*}
		To ensure the assumptions, we again use approximations to $z_2/z_1$ and $z_4/z_3$ as in Sec.\ref{ap1}.
	
	\subsubsection{\textbf{Approximation 3}} Here, the term $\br{RD\br{\w}}^2$ is expanded and each individual indicator is replaced by its corresponding approximation, i.e., we replace $\indic_{z>0}$ and $-\indic_{z>0}$ by $\max(0,1+z)$ and $-\min(1,z)$, respectively, rendering the complete constraint convex. Hence the constraint function can be written as in the form \eqref{mainProb} as
		\begin{align*}
		\vl(\vz) = \br{\frac{z_2}{z_1}}^2\! + \br{\frac{z_5}{z_4}}^2 \! + 1 - c_3^2\tau^2 -  2\frac{z_3}{z_1} - 2\frac{z_6}{z_4} + 2\frac{z_2 z_5}{z_1 z_4}
		\end{align*}
		\begin{align*}
		\h(\w;\mathbf{x},y,S) &= \bigg[\indic_{S=s^{+}},\indic_{S=s^{+}}\max\br{0,1+\w^T\x},\indic_{S=s^{+}}\min\br{1,\w^T\x}, \indic_{S=s^{-}},\\
		&\hspace{5mm}\indic_{S=s^{-}}\max\br{0,1\!-\!\w^T\x},\indic_{S=s^{-}}\min\br{1,-\w^T\x} \bigg]
		\end{align*}
		where $c_3$ is a tuning parameter. Inner function $\h(.)$ is similar to the one in Sec.\ref{ap1} and hence all the assumptions related to it are satisfied. For the outer function, we again use the approximations as in Sec. \ref{ap1}, which makes it smooth with bounded gradient, thus satisfying the assumptions in Sec. \ref{assump}. 
	
	\begin{figure*}
		\begin{subfigure}[b]{0.5\textwidth}
			\centering
			\includegraphics[height = 0.7\textwidth, width = \textwidth]{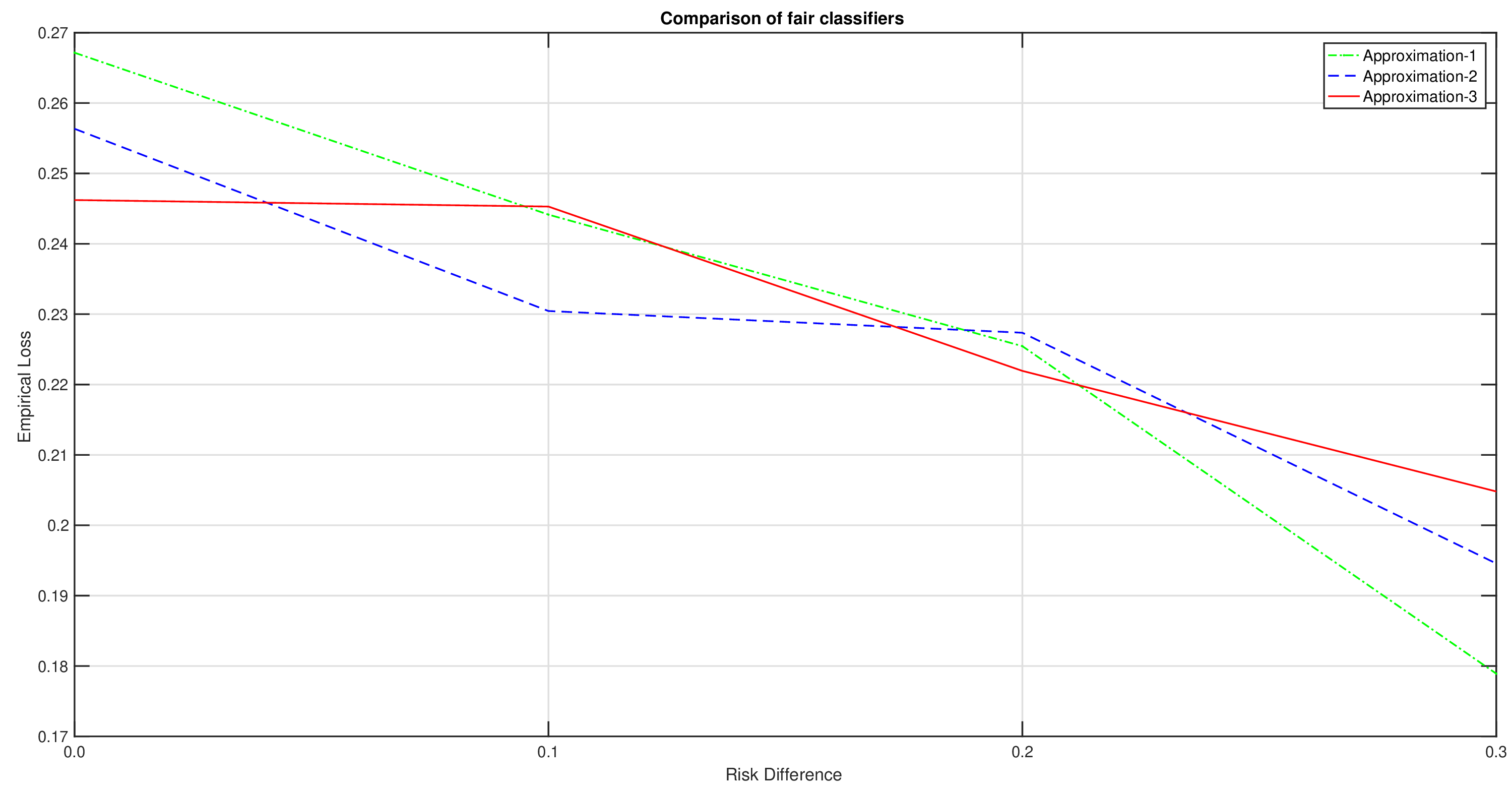}
			\caption{Comparison of fair classifiers}
			\label{fairRes1}
		\end{subfigure}
		\begin{subfigure}[b]{0.5\textwidth}
			\centering
			\includegraphics[height = 0.7\textwidth, width = \textwidth]{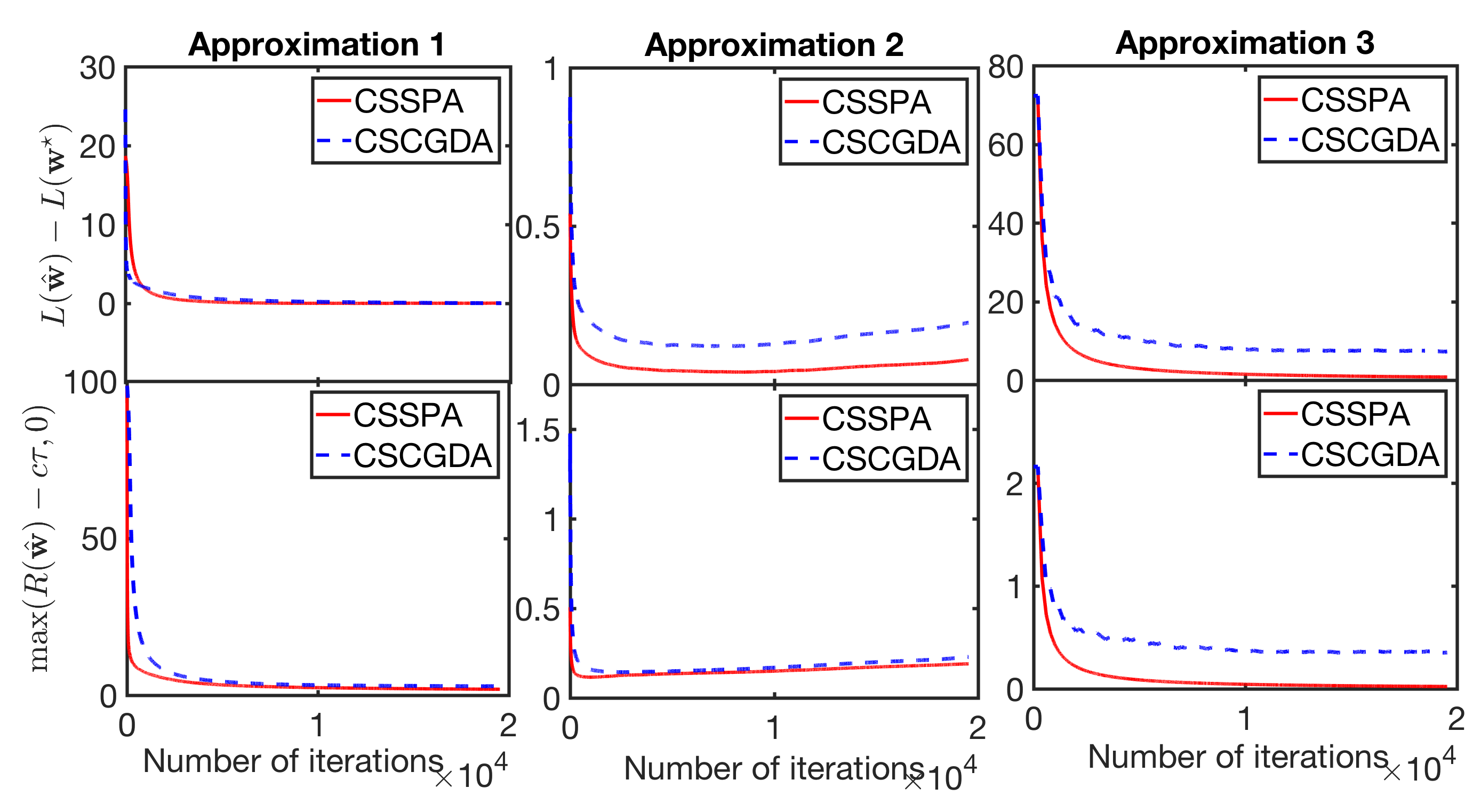}
			\caption{Convergence results}
			\label{fairRes2}
		\end{subfigure}
		\caption{Fair classifier results}
		\label{fairRes}
	\end{figure*}	  
	
	For the experiments, we consider the US Adult income dataset \cite{Adult} available in the UCI Machine Learning Repository \cite{Dua:2019}. The dataset contains a total of 48,842 instances with the features for gender, race, education, occupation, working hours, etc. from the 1994 US Census. To each entry, we assign a label 1 if the income of adult is greater than or equal to \$50,000 per annum and zero otherwise. The sensitive feature taken into account here is the gender. All the categorical features are converted into one-hot representations before training. Our goal here is to discern the label of an individual by looking at their features, while remaining fair to individuals across different genders. The three classifiers are learned with three different variants of the constraint as expressed in approximations.  All the simulations are run in MATLAB. We consider $10\%$ of the whole data for tuning various parameters. By fixing $\mu$, $c_1$, $c_2$, and $c_3$ at some random values, we first tune the constant $K$ and choose the step sizes $\alpha_{t}$, $\beta_{t}$ and $\delta_t$. Later, we tune the other parameters $\mu$, $c_1$, $c_2$, and $c_3$ to maximize the classification accuracy for each value of $\tau$ in Table \ref{tabTun}. Note that the values in Table \ref{tabTun} are obtained after solving the constrained optimization problem using an offline algorithm.
	
	\begin{table}
		\centering
		\begin{tabular}{|l|l|l|l|l|}
			\hline
			$\mathbf{\mu}$ & $\mathbf{c_1}$  &  $\mathbf{c_2}$ & $\mathbf{c_3}$ & $\mathbf{\tau}$ \\ \hline
			1     & $1.6706\times 10^{-5}$ &1.3270      & 1.3270  & 0.3 \\ \hline
			1  & $1.0048\times 10^{-5}$ &   2.0047   & 7.9810  & 0.25\\ \hline
			1  & $0.0016$ &  0.3972   & 6.2946  & 0.2\\ \hline
			1  & $4.2064\times 10^{-4}$ &  1.6746  & 1.6746  & 0.15\\ \hline
			1  & $3.9811\times 10^{-4}$ &    1.9953  &7.9433   & 0.1\\ \hline
			1  & $0.0032$  &  0.3991   & 3.9905  & 0.05\\ \hline
		\end{tabular}
		\caption{Tuning parameters} \label{tabTun} 
	\end{table}
		
		For the results, we first present the trade-off between the risk difference and the classification loss achieved by the three approximations. All the classifiers are learned on the training data and the final graph in Fig. \ref{fairRes1} is produced on testing data of 5-fold cross-validation data with $39072$ instances.  The empirical values of loss and risk difference are quantified on the testing data and are shown in Fig. \ref{fairRes1}. The trade-off curve is achieved by adjusting the parameters $\mu$, $c_1$, $c_2,$ and $c_3$ appropriately as mentioned in Table \ref{tabTun}. As evident from the Fig. \ref{fairRes1}, the classifiers of all three approximations are able to achieve much lower possible risk difference values. Further, we observe that, approximation 2,3 generate better empirical loss than approximation 1 at lower risk difference values, which may be due to the higher estimation errors. Hence, the proposed classifiers are handy and can be tuned appropriately in oder to produce not more than the designer specified bias.
		
		Next let us take a look at the evolution of the objective and constraint functions over the number of iterations for all three approximations. The performance of the proposed CSSPA algorithm is compared with the CSCGDA algorithm from \cite{thomdapu2019optimal}. To this end, we first solve the problem in an offline manner using the entire dataset. To this end, we use the trust-region reflective algorithm implemented within the `fmincon' function in MATLAB. Fig. \ref{fairRes2} shows the (running averages of the) optimality gap and the constraint violation against the number of iterations. As evident from the figure, the proposed algorithm beats CSCGDA in terms of the convergence rate. 
		
		\subsubsection{Fair SpAM}\label{numspam}
		Here, we apply CSSPA algorithm to solve SpAM problem with fairness constraints (see Sec. \ref{spamsec}). To this end, define $\zeta_1$, $\zeta_2$, $\zeta_3$, $\ldots$, $\zeta_q$ as the known basis functions of $\cH_j$, so that $h_j(x) = \sum_{k=1}^q w_k \zeta_k(x)$ where $w_k$ are the coefficients. Hence the optimization problem boils down to finding the set of coefficients corresponding to each basis function as 
		\begin{align}\label{fairspamcoeff}
		&\min_{\va \in \bR^{d\times q}} \E{\br{y - \sum_{j=1}^d \sum_{k=1}^q w_{jk}\zeta_k(x_j)}^2}+ \mu\sum_{j=1}^d\sqrt{ \E{\br{\sum_{k=1}^q w_{jk}\zeta_k(x_j)}^2}}\nonumber\\
		&\text{s.t. }-\tau \leq \E{s \sum_{j=1}^d \sum_{k=1}^q w_{jk}\zeta_k(x_j)} -\E{s}\E{\sum_{j=1}^d \sum_{k=1}^q w_{jk}\zeta_k(x_j)} \leq \tau.  
		\end{align}
		Now we can write the objective function of the problem \eqref{fairspamcoeff} in the required form as
		\begin{align*}
		\g(\w; \x_i,y_i) &=  \Bigg[\br{y_i - \sum_{j=1}^d\sum_{k=1}^q w_{jk}\zeta_k(x_{ij})}^2, \br{\sum_{k=1}^q w_{1k}\zeta_k(x_{i1})}^2,\ldots,\br{\sum_{k=1}^q w_{dk}\zeta_k(x_{id})}^2\Bigg]^\T\\
		f(\vz) &= z_1 + \mu\sum_{j=1}^d\sqrt{z_{j+1}},
		\end{align*}
		and with linear constraints. Since all data points and basis functions are finite valued, we can ensure  that $\norm{w}^2 \leq B_w$, $\norm{\x}^2 \leq B_x$, $y^2 \leq B_y$, $s^2 \leq B_s$, and $\zeta_k(\x_{ij})^2 \leq B_\zeta$ for all $\x,i,j$, and $k$. 
		
		\begin{figure}
			\centering
			\includegraphics[height = 0.35\textwidth, width = 0.49\textwidth]{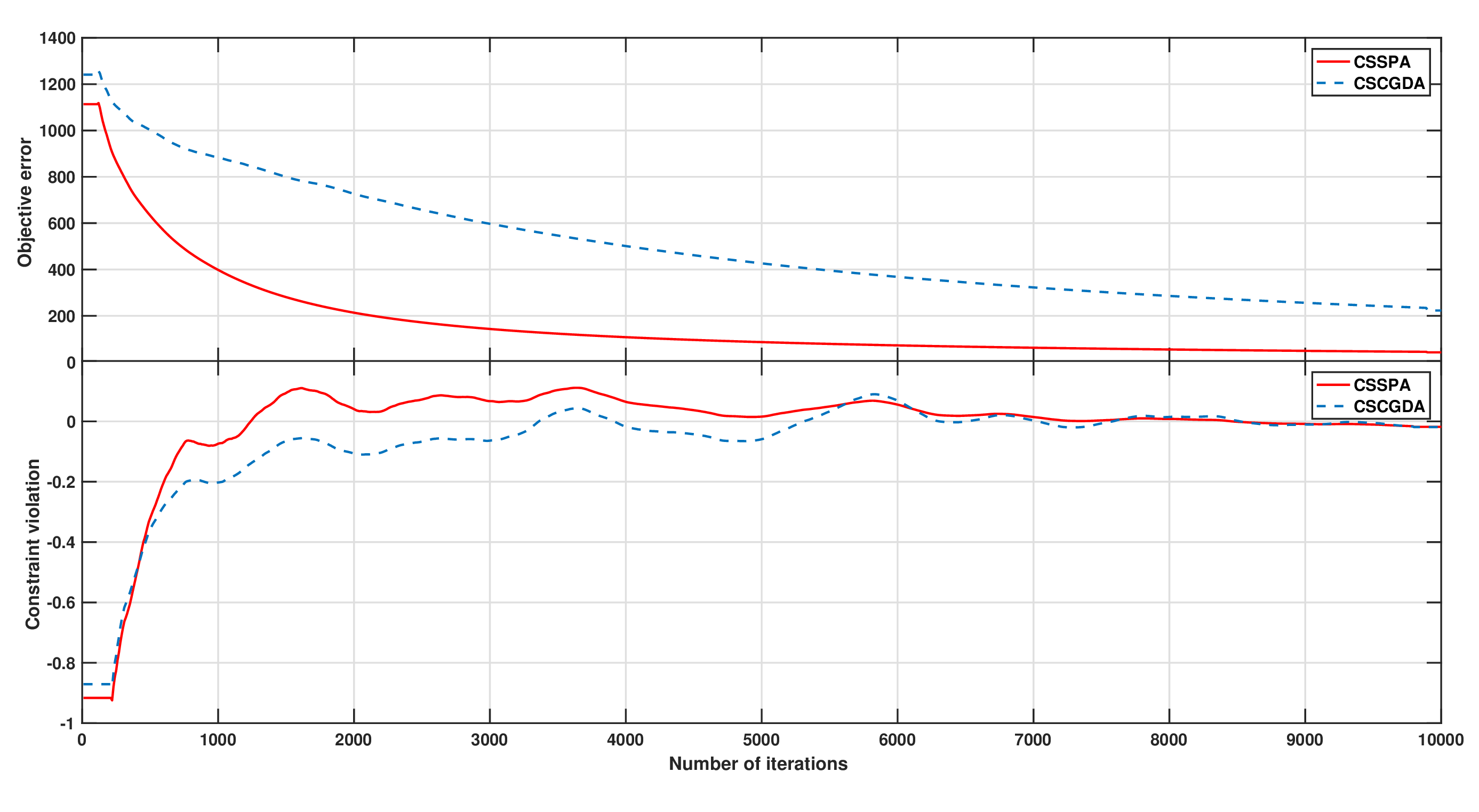}
			\caption{Fair Spam Results}
			\label{fairSpam}
		\end{figure}
		
		We test this application on a synthetic dataset. The dataset consists of 20,000 i.i.d. datapoints  generated from a uniform $\bs{0,1}$ distribution with 30 features each. We considered a model $y_i = \sum_{j=1}^{30} h_j(x_{ij}) +\epsilon_i$ for simulations. The first 4 feature functions given by  $h_1(x_{i1}) =r_1 x_{i1}^4$, $h_2(x_{i2}) = r_2x_{i2}^3$, $h_3(x_{i3})=  r_3x_{i3}^2$, $h_4(x_{i4}) = r_4x_{i4}$ are the only non-zero functions, where $r_1$, $r_2$, $r_3$, and $r_4$ are coefficients randomly selected from $\bs{0,1}$. The sensitive feature is a linear transformation of the first feature, i.e., $s_i = a\x_{i1} + b$, where $a,b$ are randomly chosen from $\bs{0,1}$. The random variable $\epsilon_i$ is Gaussian with mean 0 and variance 0.1. The parameter $\mu$ is tuned by 5-fold cross-validation. In constraints we take $\tau = 0.5$.  We have also implemented the CSCGDA method proposed in \cite{thomdapu2019optimal} and results are shown in Fig.\ref{fairSpam}. To evaluate the performance of algorithms, we found the optimal solution by using `fmincon' operator in MATLAB and use it to calculate the optimality gap. The step sizes are taken as mentioned in Table \ref{tab:title}. The experiment has been repeated over $100$ epochs using Monte Carlo method. The results in Fig.\ref{fairSpam}  hence are averaged. It can be seen that the optimality gap and the constraint violation converge at a faster rate to zero for CSSPA, as compared to the method from \cite{thomdapu2019optimal}.

		\section{Conclusion} \label{concl}	
		The constrained stochastic optimization problem with compositional objective and constraint functions is studied. The compositional form obviates the use of classical stochastic gradient-based methods that require unbiased gradient estimates. Likewise, the stochastic compositional gradient descent (SCGD) method cannot be used as it is meant for problems where projection over the constraint set can be easily computed. Hence a variant of the Arrow-Hurwicz saddle point algorithm is proposed where the expectation functions are tracked separately using auxiliary variables as estimated in the SCGD algorithm. A complete analysis with more general proof of convergence for convex problems is provided. The sample complexity of the proposed primal-dual algorithms is $\cO\br{T^{-1/4}}$ for the optimality gap after $T$ iterations while ensuring zero constraint violation. The derived rates here significantly outperform the best-known results for constrained stochastic compositional problems. Finally, the proposed algorithm is tested on two practical applications of classification and regression problems with enforced fairness constraints and is shown to be superior to the state-of-the-art algorithm.  

		\appendices
		\section{}\label{allproofs}
		This appendix contains detailed proofs of the theorems \ref{althm},\ref{bcthm} and some intermediate lemmas that we state in the subsequent sections.
		This appendix contains detailed proofs of the theorems \ref{althm},\ref{bcthm} and some intermediate lemmas that we state in the subsequent sections.
			\section{Preliminary Results}
			We begin with deriving some preliminary results. Consider the augmented Lagrangian for the surrogate problem \eqref{sur},
			\begin{align}\label{surlag}
			\Lth\br{\x,\lam,\alpha,\delta} = F(\x) + \sum_{j=1}^J\lamb_j\br{L_j(\x) + \theta} - \frac{\alpha\delta}{2}\norm{\lam}^2.
			\end{align}
			To ensure that Slater's condition is satisfied for \eqref{sur}, we will require that $\theta < \sigma_0$. The analysis proceeds by first bounding the optimality gap and the constraint violation by expressions that are functions of $\theta$. At the final step, $\theta$ will be chosen so as to ensure that the constraint violation is zero. 
			
			Different from the analysis in \cite{wang2017stochastic} and its variants, the bounds here will contain $\norm{\lam_t}^2$ term on the right. As no assumption is made on the boundedness of $\lam_t$, these bounds are not equivalent to those in \cite{wang2017stochastic}. Instead, we will follow the approach of \cite{koppel2017proximity} wherein $\delta$ must be chosen so as to ensure that $\norm{\lam_t}^2$ does not become too large. We begin with stating the following preliminary Lemma. For the analysis purpose, we define $\cF_t$ as the sigma algebra formed by random samples observed till time $t-1$, i.e.,
			\begin{align}
			\bc{\xi_1,...,\xi_{t-1}, \zeta_1,...,\zeta_{t-1}, \phi_1,...,\phi_{t-1}, \psi_1,...,\psi_{t-1}}.
			\end{align} 
			
			\begin{lemma} \label{lempre}
				Under all the Assumptions in Sec. \ref{assump},
				\begin{enumerate}
					\item from the primal variable update \eqref{primu}, it holds that 
					\begin{align} 
					&\Ec{\norm{\x_{t+1}-\x_t}^2} \leq 2\alpha_t^2\br{C_gC_f + JC_hC_{\ell}\norm{\lam_t}^2}. \label{iteb} \hspace{0.5cm} 
					\end{align}
					\item From the dual variable update \eqref{duau}, it holds that for all $1 \leq j \leq J$, 
					\begin{align}
					&\Ec{\br{\ell_j\br{\w_{t+1}:\psi_t} - \alpha_t\delta_t\lamb_{j,t} + \theta}^2} \nonumber \\
					&\hspace{10mm}\leq 4C_{\ell}\Ec{\norm{\w_{t+1}-\hb(\x_t)}^2} + 4B_\ell + 4\theta^2 + 4\alpha_t^2 \delta_t^2\E{\br{\lamb_{j,t}}^2\lvert \cF_t} \hspace{0.5cm} \label{itedb}
					\end{align}	
					\item The auxiliary variable updates \eqref{auxu1} and \eqref{auxu2} yield the bounds:
					\begin{align}
					\Ec{\norm{\y_{t+1} - \gb\br{\x_t}}^2 } &\leq \br{1 - \beta_t} \norm{\y_t - \gb\br{\x_{t-1}}}^2 +2V_g\beta_t^2 + \frac{C_g}{\beta_{t}}\norm{\x_t-\x_{t-1}}^2 \label{trk1}\\
					\Ec{\norm{\w_{t+1} - \hb\br{\x_t}}^2} &\leq \br{1 - \beta_t}\norm{\w_t - \hb\br{\x_{t-1}}}^2 +2V_h\beta_t^2  + \frac{C_h}{\beta_{t}}\norm{\x_t-\x_{t-1}}^2  \label{trk2} 
					\end{align}
				\end{enumerate}
			\end{lemma}   
			The results in \eqref{iteb} and \eqref{itedb} bound the difference between consecutive primal and dual iterates. Of these, \eqref{iteb} follows from the boundedness of the gradients (Assumptions \eqref{ainner} and \eqref{aoutergra}) and the use of norm inequalities. Likewise, \eqref{itedb} follows from Assumption \eqref{aconsb}. The recursive relationships in  \eqref{trk1} and \eqref{trk2} characterize the tracking properties of the auxiliary variables and follow similarly as in \cite[Lemma 2]{wang2017stochastic}, except for the presence of the  $\norm{\lam_{t}}^2$ term. As stated earlier, since $\lam_t$ is not assumed to be bounded, the subsequent analysis will be different from \cite{wang2017stochastic}. For the sake of completeness, following contains the detailed proof.
			\begin{proof}[Proof of \eqref{iteb}]
				Since $\x_t \in \cX$, it follows from the non-expansiveness of the projection operator that
				\begin{align}
				&\norm{\x_{t+1}-\x_t}^2 \leq \alpha_t^2\norm{\nabla \g\br{\x_t;\xi_t}\nabla f\br{\y_{t+1};\zeta_t} + \sum_{j= 1}^J\lamb_{j,t}\nabla \h\br{\x_t;\phi_t}\nabla \ell_j\br{\w_{t+1};\psi_t} }^2\\
				&\leq 2\alpha_t^2 \norm{\nabla \g\br{\x_t;\xi_t}}^2\norm{\nabla f\br{\y_{t+1};\zeta_t}}^2 +2\alpha_t^2 \norm{\lam_t}^2\sum_{j= 1}^J\norm{\nabla \h\br{\x_t;\phi_t}}^2\norm{\nabla \ell_j\br{\w_{t+1};\psi_t} }^2 \label{caclem1}
				\end{align}
				where we have used the inequality $\norm{\vz_1+\vz_2}^2\leq 2\br{\norm{\vz_1}^2+\norm{\vz_2}^2}$ and the triangle inequality. Taking conditional expectation in \eqref{caclem1} given $\cF_t$, and using the facts that $\xi_t$ is independent of $\zeta_t$ and that $\phi_t$ is independent of $\psi_t$, we obtain
				\begin{align}
				\Ec{\norm{\x_{t+1}-\x_t}^2}&\leq 2\alpha_t^2 \Ec{\norm{\nabla \g\br{\x_t;\xi_t}}^2}\Ec{ \norm{\nabla f\br{\y_{t+1};\zeta_t}}^2} \nonumber\\
				&+ 2\alpha_t^2 \norm{\lam_t}^2\sum_{j= 1}^J\Ec{\norm{\nabla \h\br{\x_t;\phi_t}}^2} \Ec{\norm{\nabla \ell_j\br{\w_{t+1};\psi_t} }^2} \label{unB}\\
				&\leq 2\alpha_t^2\br{C_gC_f + JC_hC_{\ell}\norm{\lam_t}^2}\label{bglem1}
				\end{align}
				which is the required result. Note that the last inequality in \eqref{bglem1} follows from the boundedness of the gradients (Assumptions \eqref{ainner} and \eqref{aoutergra}).
			\end{proof}
			
			\begin{proof}[Proof of \eqref{itedb}]
				We begin with using the triangle inequality for the term within the squares on the left-hand side of \eqref{itedb} as follows
				\begin{align}
				\abs{\ell_j\br{\w_{t+1};\psi_t} + \theta  + \alpha_t\delta_t\lamb_{j,t}} &= \abs{\ell_j\br{\w_{t+1};\psi_t} - \ell_j(\hb(\x_t);\psi_t) + \ell_j(\hb(\x_t);\psi_t) + \theta  + \alpha_t\delta_t\lamb_{j,t}}\nonumber\\
				&\hspace{-1cm}\leq \abs{\ell_j\br{\w_{t+1};\psi_t} - \ell_j(\hb(\x_t);\psi_t)} + \abs{\ell_j(\hb(\x_t);\psi_t)} + \theta + \alpha_t\delta_t\abs{\lamb_{j,t}} \label{ppla2}
				\end{align}
				for $1\leq j \leq J$. Squaring \eqref{ppla2}. and bounding the cross-terms, we obtain
				\begin{align}
				&\br{\ell_j\br{\w_{t+1};\psi_t} + \theta  + \alpha_t\delta_t\lamb_{j,t}}^2 \nonumber \\
				&\hspace{10mm}\leq 4\br{\ell_j\br{\w_{t+1};\psi_t} - \ell_j(\hb(\x_t);\psi_t)}^2 + 4 \br{\ell_j(\hb(\x_t);\psi_t)}^2+ 4\theta^2 + 4\alpha_t^2\delta_t^2\br{\lamb_{j,t}}^2.
				\end{align}
				Finally, taking conditional expectation given $\cF_t$ and using the bounds in Assumptions \eqref{aoutersm} and \eqref{aconsb}, we obtain
				\begin{align}
				&\E{\br{\ell_j\br{\w_{t+1};\psi_t} + \theta  + \alpha_t\delta_t\lamb_{j,t}}^2\lvert \cF_t}\nonumber \\
				&\hspace{10mm} \leq 4C_{\ell}\Ec{\norm{\w_{t+1}-\hb(\x_t)}^2} + 4B_\ell + 4\theta^2 + 4\alpha_t^2 \delta_t^2\E{\br{\lamb_{j,t}}^2\lvert \cF_t}\label{itedb2}
				\end{align}
				which is the required result.
			\end{proof}
			
			\begin{proof}[Proof of \eqref{trk1} and \eqref{trk2}]
				Define $\ve_t := \br{1- \beta_t}\br{\gb\br{\x_t} - \gb\br{\x_{t-1}}}$ and observe that it is bounded due to the continuity of $\gb$ (Assumption \eqref{ainner}) as
				\begin{align}
				\norm{\ve_t}  &\leq     \br{1- \beta_t}\sqrt{C_g}\norm{\x_t - \x_{t-1}}. 
				\end{align}
				From the definition of $\ve_t$, we can write,
				\begin{align}\label{intlem1}
				\y_{t+1} - \gb\br{\x_t} + \ve_t &= \br{1- \beta_t}\br{\y_t - \gb\br{\x_{t-1}}} + \beta_t\br{\g\br{\x_t;\xi_t} - \gb\br{\x_t}}.   
				\end{align}
				Here observe that $\Ec{\g\br{\x_t;\xi_t} - \gb\br{\x_t}} = \Ec{\br{\g\br{\x_t;\xi_t} - \gb\br{\x_t}}^\T\br{\y_t - \gb\br{\x_{t-1}}}} = 0$. Therefore, squaring \eqref{intlem1} and taking conditional expectation given $\cF_t$, we obtain
				\begin{align}
				\Ec{\norm{\y_{t+1} - \gb\br{\x_t} + \ve_t}^2}&=\br{1- \beta_t}^2\norm{\y_t - \gb\br{\x_{t-1}}}^2 + \beta_t^2 \Ec{\norm{\g\br{\x_t;\xi_t} - \gb\br{\x_t}}^2} \nonumber\\
				&\leq \br{1- \beta_t}^2\norm{\y_t - \gb\br{\x_{t-1}}}^2 + \beta_t^2V_g \label{meanlem1}
				\end{align}
				where we have used the bound in Assumption \eqref{ainner}. Next, the Peter-Paul inequality implies that 
				\begin{align}
				\norm{\y_{t+1} - \gb\br{\x_t}}^2 &\leq \br{1 + \beta_t}\norm{\y_{t+1} - \gb\br{\x_t} + \ve_t}^2 +\br{1 + \frac{1}{\beta_t}}\norm{\ve_t}^2\\
				\Rightarrow \Ec{\norm{\y_{t+1} - \gb\br{\x_t}}^2} &\leq \br{1 + \beta_t}\br{1- \beta_t}^2\norm{\y_t - \gb\br{\x_{t-1}}}^2 + V_g\beta_t^2\br{1 + \beta_t} \nonumber\\
				&\hspace{5mm}+ \frac{1-\beta_t^2}{\beta_t}C_g\norm{\x_t - \x_{t-1}}^2\\
				&\leq \br{1- \beta_t}\norm{\y_t -\gb\br{\x_{t-1}}}^2 + 2V_g\beta_t^2 + C_g\beta_t^{-1}\norm{\x_t - \x_{t-1}}^2
				\end{align}
				where we have used the fact that $\beta_t \leq 1$. The inequality in \eqref{trk2} can also be derived in the similar fashion.
			\end{proof}

			We also state the following preliminary result that bounds the optimality gap of \eqref{sur}. 
		
			\begin{lemma}\label{zc} For $0 \leq \theta \leq \sigma_0/2$, it holds that
				\begin{align*}
				F(\x^\theta)\! -\! F(\xs)\! \leq\! \theta\frac{2\sqrt{C_fC_g}D_x}{\sigma_0}, \hspace{3mm}\norm{\lam^\theta}\! \leq\! \frac{2\sqrt{JC_fC_g}D_x}{\sigma_0},
				\end{align*}
				where $\lam^\theta$ is the dual optimal point of the  problem \eqref{sur}.
			\end{lemma}
			Lemma \ref{zc} follows from standard duality theory arguments, and its full proof can be found in \cite[Appendix B]{thomdapu2019optimal}.
			
			Next, we give details on choices of various parameters $\alpha_{t}, \beta_t$ and $\delta_t$ through the following lemma which will be utilized in proving convergence results. 
				\begin{lemma}\label{chLem}
					Suppose the step sizes are selected as
					\begin{itemize}
						\item $\alpha_t = \alpha_0T^{-a}$, $\beta_t = \beta_0T^{-b}$, $\forall t$,
						\item $\alpha_t = \alpha_0t^{-a}$, $\beta_t = \beta_0t^{-b}$, 
					\end{itemize}
					where $0<a<1$, $0<b<1$, $b\leq a$ and $a \leq 2b$. For some constant $K>0$, if we choose $\alpha_0 = 1/7K$, $\beta_0 = 1$ and $ \delta_t =  4K\br{1 + \frac{1}{\beta_{t}} + \frac{1}{\beta_{t+1}}}$, then it holds
					\begin{align*}
					K\alpha_t^4\delta_t^2 - \frac{\alpha_t^2\delta_t}{2} +\! K\br{\alpha_{t}^2\! +\! \frac{\alpha_{t}^2}{\beta_{t}}\! +\! \frac{\alpha_{t}^2}{\beta_{t+1}}} \leq 0.
					\end{align*}
				\end{lemma}
				\begin{proof}
					We know for $(b-\sqrt{b^2-4ac})/2a \leq x \leq (b+\sqrt{b^2-4ac})/2a$, it holds $ax^2-bx+c \leq 0$. Now consider the quadratic equation in $\delta_t$ as
					\begin{align}\label{qudD}
					K\alpha_t^4\delta_t^2 - \frac{\alpha_t^2\delta_t}{2} +\! K\br{\alpha_{t}^2\! +\! \frac{\alpha_{t}^2}{\beta_{t}}\! +\! \frac{\alpha_{t}^2}{\beta_{t+1}}}.
					\end{align}
					For
					\begin{align} \label{reg}
					\frac{1}{4K\alpha_{t}^2}\!\br{\!1\!-\!\sqrt{1-16K^2A}} \!\leq\! \delta_t\! \leq\! \frac{1}{4K\alpha_{t}^2}\!\br{1\!+\!\sqrt{1-16K^2A}}
					\end{align}
					where $A = \alpha_{t}^2 + \frac{\alpha_{t}^2}{\beta_{t}} + \frac{\alpha_{t}^2}{\beta_{t+1}}$, we can say that the expression in \eqref{qudD} is non-positive. Since $1-\sqrt{1-x} \leq x$ for $0\leq x\leq 1$, the choice of $\delta_t = 4K\br{1 + \frac{1}{\beta_{t}} + \frac{1}{\beta_{t+1}}}$ satisfies 
					\begin{align*}
					\frac{1}{4K\alpha_{t}^2}\!\br{1\!-\!\sqrt{1-16K^2A}} \!\leq\! \delta_t.
					\end{align*}
					Now, if we prove $\delta_t \leq 1/(4K\alpha_{t}^2)$, it is sufficient to conclude that, the choice of $\delta_t$ stays in the interval as specified in \eqref{reg} and $1-16KA \geq 0$. Consider
					\begin{align}\label{int}
					4K\alpha_{t}^2\delta_t &= 16K^2\br{\alpha_{t}^2 + \frac{\alpha_{t}^2}{\beta_{t}} + \frac{\alpha_{t}^2}{\beta_{t+1}}}\nonumber\\
					&= 16K^2\br{\alpha_0^2t^{-a} + \frac{\alpha_0^2}{\beta_0} t^{b-2a} + \frac{\alpha_0^2}{\beta_0} \frac{(t+1)^b}{t^{2a}}}.
					\end{align}
					By substituting $\alpha_0 = 1/7K$ and $\beta_0 = 1$ in \eqref{int}, we write
					\begin{align*}
					4K\alpha_{t}^2\delta_t \leq \frac{1}{3}\br{t^{-a} + t^{b-2a} + \frac{(t+1)^b}{t^{2a}} }.
					\end{align*}
					Since $0<a<1$, $0<b<1$, $b\leq a$ and $a \leq 2b$, we conclude the proof by saying
					\begin{align*}
					4K\alpha_{t}^2\delta_t \leq 1.
					\end{align*}
			\end{proof}
			
			\section{Intermediate Results}
			We are now ready to derive the key lemmas relevant to the current proof. The bounds in the subsequent lemmas are stated using the big-$\mathcal{O}$ notation, with the implicit understanding that the underlying constants depend only on the problem parameters $C_g$, $C_h$, $V_g$, $V_h$, $C_f$, $C_{\ell}$, $L_f$, $L_{\ell}$, $B_{\ell}$, $D_x$, $D_y$, $D_z$, $J$ and initialization of terms $\x_1,\y_1,\w_1,\vlambda_1$ in Algorithm \ref{alg1}. We use the fact that the step-size parameters $\alpha_t$ and $\beta_t$ are non-increasing, i.e., $\alpha_t \geq \alpha_{t+1}$ and $\beta_t \geq \beta_{t+1}$ according to the statements of Theorems \ref{althm}, \ref{bcthm}. 
			
			The following lemma is the first key result that glues the objective error and the constraint violation within a single inequality.
			\begin{lemma} \label{lagl}
				Let $\xt \in \cX$ be any feasible solution to \eqref{sur} and let $\lt \geq 0$. If the step sizes are selected such that $\alpha_t \leq \beta_t$, then we have the bound
				\begin{align}
					\sum_{t=1}^T&\alpha_{t}\E{F\br{\x_t} - F\br{\xt} + \sum_{j= 1}^J\lt_j\br{L_j(\x_t) + \theta } } \nonumber\\
					&\leq\O\br{1 +  \sum_{t=1}^T\bs{\alpha_{t}^2 + \beta_t^2+ \frac{\alpha_{t}^2}{\beta_{t}} + \frac{\alpha_{t}^2}{\beta_{t+1}}}} \br{1+\norm{\lt}^2} \label{lams}
					\end{align}
			\end{lemma}
			The idea of bounding the left-hand side is borrowed from \cite[Lemma 1]{koppel2015saddle}. However, the presence of non-linear functions of expectations on the left must be separately addressed using the techniques from \cite{wang2017stochastic}. The proof of Lemma \ref{lagl} requires establishing five preliminary lemmas. In Lemmas \ref{primopt} and \ref{lagprim}, we use the primal update \eqref{primu} and convexity of $\Lth$ in $\x$ to establish a bound on $\sum_t \Lth(\x_t,\lam_t,\alpha_t,\delta_t)-\Lth(\xt,\lam_t,\alpha_t,\delta_t)$. Likewise, in Lemmas \ref{dualopt} and \ref{lagdual}, we use the dual update \eqref{duau} and strong concavity of $\Lth$ with respect to $\lam$ to establish bound on $\sum_t \Lth(\x_t,\lt) - \Lth(\x_t,\lam_t)$. Lemma \ref{lagl} would then follow by adding the results in Lemmas \ref{lagprim} and \ref{lagdual} and simplifying. 
			
			We begin with bounding the average decrement in $\norm{\x_{t+1}-\xt}$ where $\xt$ is a feasible point of \eqref{sur}. 
			\begin{lemma}\label{primopt}
				For any $\xt$ feasible for \eqref{sur}, the following inequality holds with probability one:
				\begin{align} \label{primopteq}
				&\Ec{\norm{\x_{t+1} - \xt}^2 } \leq \norm{\x_t - \xt}^2 +  \frac{\alpha_t^2}{\beta_{t}}L_f^2C_gD_x - 2\alpha_t\br{\Lth\br{\x_t,\lam_t,\alpha_{t},\delta_t} - \Lth\br{\xt,\lam_t,\alpha_{t},\delta_t}} \nonumber\\ 
				&\hspace{5mm}+ \beta_{t}\Ec{ \norm{\gb\br{\x_t} - \y_{t+1}}^2 +  \norm{\hb\br{\x_t} - \w_{t+1}}^2} + J\br{\frac{\alpha_t^2}{\beta_{t}}L_{\ell}^2C_hD_x + 2\alpha_t^2C_hC_{\ell}}\norm{\lam_t}^2 \nonumber \\
				&\hspace{5mm} +2\alpha_t^2C_fC_g. 
				\end{align}
			\end{lemma}
			\begin{proof}
				Since $\x_t \in \cX$, it follows from the non-expansiveness of the projection operation that
				\begin{align}
				&\norm{\x_{t+1} - \xt}^2 \nonumber\\
				&\hspace{3mm}\leq \norm{\x_t - \xt -\alpha_t\br{\nabla \g\br{\x_t;\xi_t}\nabla f\br{\y_{t+1};\zeta_t} + \sum_{j= 1}^J\lamb_{j,t}\nabla \h\br{\x_t;\phi_t}\nabla \ell_j\br{\w_{t+1};\psi_t}}}^2\\
				&\hspace{3mm}= \norm{\x_t - \xt}^2 + \alpha_t^2\norm{\nabla \g\br{\x_t;\xi_t}\nabla f\br{\y_{t+1};\psi_t} + \sum_{j= 1}^J\lamb_{j,t}\nabla \h\br{\x_t;\phi_t}\nabla \ell_j\br{\w_{t+1};\psi_t}}^2 \nonumber\\
				&\hspace{6mm} - 2\alpha_t\br{\x_t - \xt}^\T\br{\nabla \g\br{\x_t;\xi_t}\nabla f\br{\y_{t+1};\zeta_t} + \sum_{j= 1}^J\lamb_{j,t}\nabla \h\br{\x_t;\phi_t}\nabla \ell_j\br{\w_{t+1};\psi_t}}. \label{lemx1}
				\end{align}
				Let us denote
				\begin{align}
				u_t &:= (\x_t-\xt)^\T\nabla \g\br{\x_t;\xi_t} \br{\nabla f\br{\gb(\x_t);\zeta_t} - \nabla f\br{\y_{t+1};\zeta_t}} \\
				v_t &:= (\x_t-\xt)^\T\sum_{j= 1}^J\lamb_{j,t}\nabla \h\br{\x_t;\phi_t}\br{\nabla \ell_j\br{\hb(\x_t);\psi_t} - \nabla \ell_j\br{\w_{t+1};\psi_t}}.
				\end{align}
				Taking conditional expectation in \eqref{lemx1} and using \eqref{iteb} from Lemma \ref{lempre}, we can write 
				\begin{align}
				&\Ec{\norm{\x_{t+1} - \xt}^2}  \leq \norm{\x_t - \xt}^2 +   2\alpha_t^2C_gC_f + 2\alpha_t^2JC_hC_\ell\norm{\lam_t}^2 + \Ec{u_t} + \Ec{v_t }  \nonumber\\
				& - 2\alpha_t\br{\x_t - \xt}^\T\Ec{\nabla \g\br{\x_t;\xi_t}\nabla f\br{\gb(\x_t);\zeta_t} + \sum_{j= 1}^J\lamb_{j,t}\nabla \h\br{\x_t;\phi_t}\nabla \ell_j\br{\hb(\x_t);\psi_t}}. \label{inL}
				\end{align}
				Recalling that $\Ec{\nabla \g\br{\x_t;\xi_t}\nabla f\br{\gb(\x_t);\zeta_t}} = \nabla F(\x_t)$ and $\Ec{\nabla \h\br{\x_t;\phi_t}\nabla \ell_j\br{\hb(\x_t);\psi_t} }=\nabla L_j(\x_t)$, and using the definition of $\Lth$ from \eqref{surlag}, we obtain
				\begin{align}
				\Ec{\norm{\x_{t+1} - \xt}^2 } &\leq \norm{\x_t - \xt}^2 + 2\alpha_t^2\br{C_gC_f + JC_hC_{\ell}\norm{\lam_t}^2} + \Ec{u_t} + \Ec{v_t} \nonumber\\
				&\hspace{1cm} - 2\alpha_t\br{\x_t - \xt}^\T\nabla_{\x}\Lth\br{\x_t,\lam_t,\alpha_{t},\delta_t}\\
				&\leq \norm{\x_t - \xt}^2 + 2\alpha_t^2\br{C_gC_f + JC_hC_{\ell}\norm{\lam_t}^2} + \Ec{u_t} + \Ec{v_t} \nonumber \\
				&\hspace{1cm}- 2\alpha_t\br{\Lth\br{\x_t,\lam_t,\alpha_{t},\delta_t} - \Lth\br{\xt,\lam_t,\alpha_{t},\delta_t}} \label{lmain} 
				\end{align}
				where we have used the convexity of $\Lth\br{\x,\lam,\alpha,\delta}$ with respect to $\x$; see Assumption \eqref{adiff}. The term $u_t$ can be bounded by using the smoothness of $f$ (Assumption \eqref{aoutersm}) and the Peter-Paul inequality as follows:
				\begin{align}
				u_t &= 2\alpha_t\br{\x_t - \xt}^\T \nabla \g\br{\x_t;\xi_t} \br{\nabla f\br{\gb(\x_t);\zeta_t} - \nabla f\br{\y_{t+1};\zeta_t}} \\
				&\leq 2\alpha_t L_f\norm{\x_{t} - \xt}\norm{\y_{t+1} - \gb(\x_t)} \norm{\nabla \g\br{\x_t;\xi_t}} \label{cach}\\
				&\leq \beta_{t} \norm{\y_{t+1} - \gb(\x_t)}^2 + \frac{\alpha_t^2}{\beta_{t}} L_f^2\norm{\x_{t} - \xt}^2\norm{\nabla \g\br{\x_t;\xi_t}}^2 \label{ppI}\\
				&\leq \beta_{t} \norm{\y_{t+1} - \gb(\x_t)}^2 + \frac{\alpha_t^2}{\beta_{t}} L_f^2D_x\norm{\nabla\g\br{\x_t;\xi_t}}^2 \label{dia}\\
				\Rightarrow ~~ \E{u_t\cond\cF_t} & \leq \beta_{t} \Ec{\norm{\y_{t+1} - \gb(\x_t)}^2 } + \frac{\alpha_t^2}{\beta_{t}} L_f^2D_xC_g, \label{uexp} 
				\end{align}
				where we have used the compactness of $\cX$ (Assumption \eqref{aslater}) in \eqref{dia} and the boundedness of the gradient (Assumption \eqref{ainner}) in \eqref{uexp}. Proceeding along similar lines and again using Assumptions \eqref{aoutersm}, \eqref{aslater}, and \eqref{ainner},  we obtain
				\begin{align}
				\Ec{v_t}  \leq \beta_{t} \Ec{\norm{\w_{t+1} - \hb(\x_t)}^2 } + \frac{\alpha_t^2}{\beta_{t}} JL_{\ell}^2D_xC_h\norm{\lam_t}^2 \label{vexp}
				\end{align}
				Substituting the expressions \eqref{uexp} and \eqref{vexp}  in \eqref{lmain}, we obtain the required result.
			\end{proof}
			
			Let initial tracking errors are denoted as $\norm{\y_{1}-\gb \br{\x_0}}^2= D_y$ and $\norm{ \w_{1}-\hb \br{\x_0}}^2= D_w$ and $D$ is denotes as $D = D_x+D_y+D_z$. Building up on the result in Lemma \ref{primopt}, we bound the total Lagrangian deviation. 
				\begin{lemma}\label{lagprim} 
					The bounds in Lemmas \ref{lempre} and \ref{primopt} imply that
					\begin{align}
					&\sum_{t=1}^T\alpha_{t}\E{\Lth\br{\x_t,\lam_t,\alpha_{t},\delta_t} - \Lth\br{\xt,\lam_t,\alpha_{t},\delta_t}} \leq D\! +\! \sum_{t=1}^T\bigg(\!2C_fC_g\!\br{C_g \!+\! C_h}\!\frac{\alpha^2_{t}}{\beta_{t+1}} \!+\! L_f^2C_gD_x\!\frac{\alpha_t^2}{\beta_{t}}\! \nonumber\\
					&\!+\! C_fC_g\alpha_t^2\!  +\! 2(V_g\!+\!V_h)\beta_{t}^2\! \bigg)\! +\! \sum_{t=1}^T\!J\bigg(\!C_hC_{\ell} \alpha_t^2\!  +\!  2C_hC_{\ell}(C_g+C_h)\frac{\alpha_{t}^2}{\beta_{t+1}}\!  +\!  L_{\ell}^2C_hD_x\frac{\alpha_t^2}{\beta_{t}}\bigg)\E{\norm{\lam_t}^2}.  
					\end{align}
				\end{lemma}
				\begin{proof}
					Let us define
					\begin{align*}
					\cI_t := \EE\Big[\!\norm{\x_{t}\!-\!\xt}^2 \!+\! \norm{\y_{t}\!-\!\gb(\x_{t-1})}^2\! +\! \norm{\w_{t}\!-\!\hb(\x_{t-1})}^2\Big].
					\end{align*}
					By taking full expectation in the result of Lemma \ref{primopt} and using \eqref{trk1} and \eqref{trk2}  from Lemma \ref{lempre}, we obtain
					\begin{align}
					&\cI_{t+1} \nonumber\\
					&\hspace{3mm}\leq\! \cI_t \!+\! 4(V_g\!+\!V_h)\beta_{t}^2 \! +\!\frac{\alpha_t^2}{\beta_{t}}JL_{\ell}^2C_hD_x\E{\norm{\lam_t}^2} \!+\! 2 \alpha_t^2 \br{C_gC_f \!+\! JC_hC_{\ell}\E{\norm{\lam_t}^2}} \!+\! L_f^2C_gD_x\frac{\alpha_t^2}{\beta_{t}} \nonumber\\ 
					&\hspace{6mm}   - 2\alpha_t\E{\Lth_t\br{\x_t,\lam_t,\alpha_{t},\delta_t} - \Lth_t\br{\xt,\lam_t,\alpha_{t},\delta_t}}  + 2\br{C_g+C_h}\beta_t^{-1} \E{\norm{\x_t - \x_{t-1}}^2} \\
					&\hspace{3mm}\leq\! \cI_t \!+\! 4(V_g\!+\!V_h)\beta_{t}^2 \! +\!\frac{\alpha_t^2}{\beta_{t}}JL_{\ell}^2C_hD_x\E{\norm{\lam_t}^2} \!+\! 2 \alpha_t^2 \br{C_gC_f \!+\! JC_hC_{\ell}\E{\norm{\lam_t}^2}} \!+\! L_f^2C_gD_x\frac{\alpha_t^2}{\beta_{t}} \nonumber\\
					&\hspace{6mm} - 2\alpha_t\E{\Lth_t\br{\x_t,\lam_t,\alpha_{t},\delta_t} - \Lth_t\br{\xt,\lam_t,\alpha_{t},\delta_t}} + 4C_gC_f\br{C_g+C_h}\alpha_{t-1}^2\beta_t^{-1}\nonumber\\
					&\hspace{6mm} + 4JC_hC_{\ell}\br{C_g+C_h}\alpha_{t-1}^2\beta_t^{-1}\E{\norm{\lam_{t-1}}^2}, \label{prefrst}
					\end{align}
					where the inequality \eqref{prefrst} holds from \eqref{iteb}. Rearranging the terms, we obtain
					\begin{align*}
					&2\alpha_{t}\E{\Lth\br{\x_t,\lam_t,\alpha_{t},\delta_t} - \Lth\br{\xt,\lam_t,\alpha_{t},\delta_t}} \nonumber\\
					&\hspace{3mm}\leq  \cI_t-\cI_{t+1} + 2C_fC_g\alpha_t^2 + 4C_fC_g\br{C_g + C_h}\frac{\alpha^2_{t-1}}{\beta_{t}}  + 4JC_hC_{\ell}(C_g+C_h)\frac{\alpha_{t-1}^2}{\beta_{t}}\E{\norm{\lam_{t-1}}^2}  \nonumber\\
					&\hspace{6mm} + L_f^2C_gD_x\frac{\alpha_t^2}{\beta_{t}} + 4V_g\beta_{t}^2 +4V_h\beta_{t}^2 + \E{\norm{\lam_t}^2} J\br{2C_hC_{\ell} \alpha_t^2  +  L_{\ell}^2C_hD_x\frac{\alpha_t^2}{\beta_{t}} }.
					\end{align*}
					Summing over $t = 1, \ldots, T$, and canceling out the telescopic terms we get the required result.
			\end{proof}
			
			Next, we derive the corresponding results for the dual variable $\lam_t$ using the updates in \eqref{duau}. 
			\begin{lemma}\label{dualopt} 
				For any $\lt \geq 0$, if the step sizes are chosen as $\alpha_t^2\leq \beta_t$ then the following inequality holds with probability one:
				\begin{align}\label{dualopteq}
				&\Ec{\norm{\lam_{t+1} -\lt}^2 } \leq \norm{\lam_t - \lt}^2\br{1+ \frac{\alpha_t^2JC_{\ell}}{\beta_{t}}} + 2\alpha_t\br{\Lth\br{\x_t,\lam_t,\alpha_{t},\delta_t} - \Lth\br{\x_t,\lt,\alpha_{t},\delta_t}} \nonumber\\
				&\hspace{10mm}+ \beta_{t}\br{1+4JC_{\ell}}\Ec{\norm{\w_{t+1} - \hb\br{\x_t}}^2} + 4J(B_{\ell}+\theta^2)\alpha_t^2 + 4\alpha_t^4\delta_t^2\norm{\lam_t}^2
				\end{align}
			\end{lemma}
			\begin{proof}
				Since $\lam_t \geq 0$, we can write
				\begin{align}
				\norm{\lam_{t+1}-\lt}^2 & \leq \sum_{j= 1}^J\br{\alpha_t\br{\ell_j\br{\w_{t+1}:\psi_t} - \alpha_t\delta_t\lamb_{j,t} + \theta} + \lamb_{j,t} - \lts_{j}}^2 \\
				& = \norm{ \lam_t - \lt}^2 + 2\alpha_t\sum_{j= 1}^J\br{\lamb_{j,t} - \lts_j}\br{\ell_j\br{\w_{t+1}:\psi_t} - \alpha_t\delta_t\lamb_{j,t} + \theta} \nonumber \\
				&\hspace{10mm} + \alpha_t^2\sum_{j= 1}^J\br{\ell_j\br{\w_{t+1}:\psi_t} - \alpha_t\delta_t\lamb_{j,t} + \theta}^2 \\
				& = \norm{ \lam_t - \lt}^2 + 2\alpha_t\sum_{j= 1}^J\br{\lamb_{j,t} - \lts_j}\br{\ell_j\br{\bar{h}\br{\x_t}:\psi_t} - \alpha_t\delta_t\lamb_{j,t} + \theta} + \tilde{u}_t\nonumber \\
				&\hspace{10mm} + \alpha_t^2\sum_{j= 1}^J\br{\ell_j\br{\w_{t+1}:\psi_t} - \alpha_t\delta_t\lamb_{j,t} + \theta}^2\label{intD}
				\end{align}
				where the term $\tilde{u}_t$ can be bounded by using the smoothness of $\lj$ (Assumption \eqref{aoutersm}) so as to yield
				\begin{align}
				\tilde{u}_t &:= 2\alpha_t\sum_{j= 1}^J\br{\lamb_{j,t} - \lts_j}\br{\ell_j\br{\w_{t+1};\psi_t} - \ell_j\br{\hb\br{\x_t};\psi_t}}\\
				&\leq 2\alpha_t\norm{\lam_t - \lt}\sqrt{\sum_{j= 1}^J\br{\ell_j\br{\w_{t+1};\psi_t} - 
						\ell_j\br{\hb\br{\x_t};\psi_t}}^2} \label{cachd}\\
				&\leq 2\alpha_t\sqrt{J}\sqrt{C_{\ell}} \norm{\lam_t - \lt}\norm{\w_{t+1} - \hb\br{\x_t}} \label{smoO}\\
				&\leq \beta_t\norm{\w_{t+1} - \hb\br{\x_t}}^2 + \frac{\alpha_t^2JC_{\ell}}{\beta_t}\norm{\lam_t -\lt}^2. \label{ppId}
				\end{align}
				Here, \eqref{cachd} follows from the Cauchy-Schwartz inequality while \eqref{ppId} follows from the Peter-Paul inequality. Taking conditional expectation given $\cF_t$ in \eqref{intD}, and recalling that $\Ec{\ell_j\br{\hb\br{\x_t}:\psi_t}} = L_j(\x_t)$, we obtain
				\begin{align}
				&\Ec{\norm{\lam_{t+1}-\lt}^2 } \hspace{5mm}\leq \norm{\lam_t - \lt}^2 + 2\alpha_t\sum_{j= 1}^J\br{\lamb_{j,t} - \lts_j}\br{L_j(\x_{t}) - \alpha_t\delta_t\lamb_{j,t} + \theta} \nonumber \\
				&\hspace{10mm} + \alpha_t^2\sum_{j= 1}^J\Ec{\br{\ell_j\br{\w_{t+1}:\psi_t} - \alpha_t\delta_t\lamb_{j,t} + \theta}^2} + \E{\tilde{u}_t\cond
					\cF_t} \\
				& \hspace{5mm} =  \norm{\lam_t - \lt}^2\br{1+ \frac{\alpha_t^2JC_{\ell}}{\beta_{t}}} + 2\alpha_t\br{\lam_t - \lt}^\T\nabla_{\lam}\Lth\br{\x_t,\lam_{t},\alpha_{t},\delta_t} \nonumber\\
				&\hspace{10mm}+ \alpha_t^2\sum_{j= 1}^J\Ec{\br{\ell_j\br{\w_{t+1}:\psi_t} - \alpha_t\delta_t\lamb_{j,t} + \theta}^2}  + \beta_t\Ec{\norm{\w_{t+1} - \hb\br{\x_t}}^2}
				\end{align}
				where we have substituted \eqref{ppId}. 
				Since $\Lth\br{\x,\lam,\alpha,\delta} $ is $\alpha_t\delta_t$ strongly-concave in $\lam$, we have that 
				\begin{align}
				&\Ec{\norm{\lam_{t+1}-\lt}^2 } \nonumber\\
				&\hspace{5mm} \leq \norm{\lam_t - \lt}^2\br{1+ \frac{\alpha_t^2JC_{\ell}}{\beta_{t}} - \alpha_t^2\delta_t}+ 2\alpha_{t}\br{\Lth\br{\x_t,\lam_t,\alpha_{t},\delta_t} - \Lth\br{\x_t,\lt,\alpha_{t},\delta_t}} \nonumber \\
				&\hspace{10mm}+ \alpha_t^2\sum_{j= 1}^J\Ec{\br{\ell_j\br{\w_{t+1}:\psi_t} - \alpha_t\delta_t\lamb_{j,t} + \theta}^2} + \beta_t\Ec{\norm{\w_{t+1} - \hb\br{\x_t}}^2} \\
				&\hspace{5mm} \leq \norm{\lam_t - \lt}^2\br{1+ \frac{\alpha_t^2JC_{\ell}}{\beta_{t}} - \alpha_t^2\delta_t}+ 2\alpha_{t}\br{\Lth\br{\x_t,\lam_t,\alpha_{t},\delta_t} - \Lth\br{\x_t,\lt,\alpha_{t},\delta_t}} \nonumber \\
				&\hspace{10mm} + \br{4J\alpha_{t}^2C_{\ell}+\beta_t}\Ec{\norm{\w_{t+1} - \hb\br{\x_t}}^2} + 4\alpha_{t}^2JB_\ell + 4\alpha_{t}^2J\theta^2 + 4\alpha_t^4 \delta_t^2\norm{\lam_{t}}^2 \label{fin} ,
				\end{align} 
				where the last inequality is followed by \eqref{itedb}. From the statement of Lemma \ref{lagl}, since $\alpha_{t} \leq \beta_{t}$, we also have $\alpha_{t}^2\leq \beta_{t}$. Therefore we can write $\\4J\alpha_{t}^2C_{\ell}\Ec{\norm{\w_{t+1} - \hb\br{\x_t}}^2} \leq 4J\beta_t C_{\ell}\Ec{\norm{\w_{t+1} - \hb\br{\x_t}}^2}$ in \eqref{fin}. 
			\end{proof}

			\begin{lemma} \label{lagdual} 
					Statements of Lemmas \ref{lempre}, \ref{dualopt} yield
					\begin{align}
					\sum_{t=1}^T\alpha_{t}&\E{\Lth\br{\x_t,\lt,\alpha_{t},\delta_t} - \Lth\br{\x_t,\lam_t,\alpha_{t},\delta_t}} \nonumber\\
					&\leq D+ \norm{\lt}^2 +2JC_{\ell}\norm{\lt}^2\sum_{t=1}^T\frac{\alpha_{t}^2}{\beta_{t}} \sum_{t=1}^T\br{2J(B_{\ell}+\theta^2)\alpha_t^2 + 2CV_h\beta_t^2 + 2CC_fC_gC_h\frac{\alpha_{t}^2}{\beta_{t+1}} }\nonumber\\
					&\hspace{5mm} +\sum_{t=1}^T \bigg(2\alpha_t^4\delta_t^2+ 2JC_{\ell}\frac{\alpha_t^2}{\beta_{t}} + 2JCC_h^2C_{\ell}\frac{\alpha_{t}^2}{\beta_{t+1}} \bigg)\E{\norm{\lam_t}^2}, 
					\end{align}
					where $C = 1+4JC_{\ell}$
				\end{lemma}
				\begin{proof}
					Let $\cJ_t := \E{\norm{\lam_t -\lt}^2} + C\E{\norm{\w_{t+1} - \hb\br{\x_t}}^2}$, where $C:= 1+4JC_{\ell}$. Using statement of Lemma $\ref{dualopt}$ and \eqref{trk2} from Lemma \ref{lempre}, we can write
					\begin{align}
					\cJ_{t+1} &\leq \cJ_t +2\alpha_t\E{\Lth\br{\x_t,\lam_t,\alpha_{t},\delta_t} - \Lth\br{\x_t,\lt,\alpha_{t},\delta_t}} +  \frac{\alpha_t^2JC_{\ell}}{\beta_t}\E{\norm{\lam_t - \lt}^2} +4CV_h\beta_t^2 \nonumber\\
					& \hspace{2mm}+ 4\alpha_t^4\delta_t^2\E{\norm{\lam_t}^2} + 4J(B_{\ell}+\theta^2)\alpha_t^2 + 2C\beta_t^{-1}C_h\E{\norm{\x_t - \x_{t-1}}^2}\\
					&\leq \cJ_t + 2\alpha_t\E{\Lth\br{\x_t,\lam_t,\alpha_{t},\delta_t} - \Lth\br{\x_t,\lt,\alpha_{t},\delta_t}} +  \frac{\alpha_t^2JC_{\ell}}{\beta_t}\E{\norm{\lam_t - \lt}^2} +4CV_h\beta_t^2 \nonumber\\
					& \hspace{2mm}+ 4\alpha_t^4\delta_t^2\E{\norm{\lam_t}^2}  + 4C\br{\alpha_{t-1}^2\beta_t^{-1}C_h\br{C_gC_f + JC_hC_{\ell}\E{\norm{\lam_{t-1}}^2}}} \! +\! 4J(B_{\ell}+\theta^2)\alpha_t^2 ,
					\end{align}
					where the last inequality follows from \eqref{iteb}. Rearranging the terms and using the inequality $\norm{\lam_t - \lt}^2 \leq \norm{\lam_t}^2 + \norm{\lt}^2$, we obtain
					\begin{align*}
					2\alpha_{t}&\E{\Lth\br{\x_t,\lt,\alpha_{t},\delta_t} - \Lth\br{\x_t,\lam_t,\alpha_{t},\delta_t}} \nonumber\\
					&\leq  \cJ_t- \cJ_{t+1} + 
					4J(B_{\ell}+\theta^2)\alpha_t^2 \!+\! 4CC_fC_gC_h\frac{\alpha_{t-1}^2}{\beta_t} + 4CV_h\beta_t^2 + 2JC_{\ell}\norm{\lt}^2\frac{\alpha_{t}^2}{\beta_{t}}   \nonumber\\
					&\hspace{2mm}  + \br{4\alpha_t^4\delta_t^2 + 2JC_{\ell}\alpha_t}\E{\norm{\lam_t}^2} + 4JCC_h^2C_{\ell}\frac{\alpha_{t-1}^2}{\beta_{t}}\E{\norm{\lam_{t-1}}^2}.
					\end{align*}
					Summing over $t=1$ to $T$, and canceling out the telescopic terms, we obtain the required result.
			\end{proof}
			
			Having established the basic results, we are ready to prove Lemma \ref{lagl}.  
				\begin{proof}[Proof of Lemma \ref{lagl}]
					From the definition of $\Lth$ in \eqref{surlag}, we have that
					\begin{align}
					&\E{\Lth\br{\x_t,\lt,\alpha_t,\delta_t} - \Lth\br{\xt,\lam_t,\alpha_t,\delta_t}} \nonumber\\
					&\hspace{3mm}=\EE\Big[\Lth\br{\x_t,\lam_t,\alpha_{t},\delta_t} - \Lth\br{\xt,\lam_t,\alpha_{t},\delta_t} + \Lth\br{\x_t,\lt,\alpha_{t},\delta_t} - \Lth\br{\x_t,\lam_t,\alpha_{t},\delta_t}\Big]\nonumber\\
					&\hspace{3mm}=\EE\bs{\sum_{j= 1}^J\bs{ \lts_j\br{L_j(\x_t) + \theta } - \lamb_{j,t}\br{L_j(\xt) + \theta}} +F\br{\x_t} - F\br{\xt} + \frac{\alpha_t\delta_t}{2}\br{\norm{\lam_t}^2 - \norm{\lt}^2}} \nonumber\\
					&\hspace{3mm}\geq\EE\bs{F\br{\x_t} - F\br{\xt} +  \sum_{j= 1}^J\lts_j\br{L_j(\x_t) + \theta } + \frac{\alpha_t\delta_t}{2}\br{\norm{\lam_t}^2 - \norm{\lt}^2}}\label{confac}.
					\end{align}
					where \eqref{confac} follows from the fact that $\lamb_{j,t} \geq 0$ and $L_j(\xt)+\theta \leq 0$ for all $1\leq j \leq J$. Now from Lemmas \ref{lagprim} and \ref{lagdual}, we can write 
					\begin{align}
					\sum_{t=1}^T&\alpha_{t}\EE\bs{F\br{\x_t} - F\br{\xt} + \sum_{j= 1}^J \lts_j \br{L_j(\x_t)  +  \theta } + \frac{\alpha_t\delta_t}{2} \br{\norm{\lam_t}^2  - \norm{\lt}^2} }\nonumber\\
					& \leq 2D\!+\!\norm{\lt}^2\!\!+\! \sum_{t=1}^TJ\Bigg(2CC_h^2C_{\ell}\frac{\alpha_{t}^2}{\beta_{t+1}}   + 2C_hC_{\ell}(C_g+C_h)\frac{\alpha_{t}^2}{\beta_{t+1}}  +2C_{\ell}\frac{\alpha_t^2}{\beta_{t}}+2\alpha_t^4\delta_t^2+ C_hC_{\ell} \alpha_t^2 \nonumber\\
					&\hspace{3mm}+ L_{\ell}^2C_hD_x\frac{\alpha_t^2}{\beta_{t}}   \Bigg)\E{\norm{\lam_t}^2} + \sum_{t=1}^T\bigg(L_f^2C_gD_x\frac{\alpha_t^2}{\beta_{t}}+ 2JC_{\ell}\norm{\lt}^2 \frac{\alpha_{t}^2}{\beta_{t}}  + C_fC_g\alpha_t^2 + 2CV_h\beta_t^2 \nonumber\\
					&\hspace{3mm} + 2C_fC_g\br{C_g + C_h}\frac{\alpha^2_{t}}{\beta_{t+1}} + 2CC_fC_gC_h\frac{\alpha_{t}^2}{\beta_{t+1}} + 2(V_g+V_h)\beta_{t}^2
					+2J(B_{\ell}+\theta^2)\alpha_t^2  \bigg) 
					\end{align}
					For the sake of brevity, we define 
					\begin{align}
					G_1 &= 2J\max\bc{2,C_hC_{\ell}(CC_h+C_g+C_h), C_{\ell} + L_{\ell}^2C_hD_x, C_hC_{\ell}} \label{delC}\\
					G_2 &= 2\max\bc{ D, L_f^2C_gD_x, JC_{\ell}, C_fC_g + J(B_{\ell}+\theta^2),C_fC_g(CC_h+C_g+C_h), V_g+V_h+CV_h}\label{bouC}
					\end{align}
					Now by interchanging the terms, we can write
					\begin{align}\label{finex}
					\sum_{t=1}^T&\alpha_{t}\EE\bs{F\br{\x_t} - F\br{\xt} + \sum_{j= 1}^J\lt_j\br{L_j(\x_t) + \theta } - \br{\frac{1}{\alpha_tT} + 2JC_{\ell}\frac{\alpha_{t}}{\beta_{t}} + \frac{\alpha_t\delta_t}{2}} \norm{\lt}^2} \nonumber\\
					&\leq \sum_{t=1}^T\bigg(\!G_1\alpha_t^4\delta_t^2 - \frac{\alpha_t^2\delta_t}{2}+ G_1\br{\alpha_{t}^2 + \frac{\alpha_{t}^2}{\beta_{t}} + \frac{\alpha_{t}^2}{\beta_{t+1}}}\!\bigg)\E{\norm{\lam_t}^2} \nonumber\\
					&\hspace{1mm}+ G_2\sum_{t=1}^T\br{\frac{1}{T} + \alpha_{t}^2 + \beta_t^2+ \frac{\alpha_{t}^2}{\beta_{t}} + \frac{\alpha_{t}^2}{\beta_{t+1}}}.
					\end{align}
					By the statement of Lemma \ref{chLem}, the first term on the RHS is negative if we choose $\delta_t = 4G_1\br{1 + \frac{1}{\beta_{t}} + \frac{1}{\beta_{t+1}}}$. Finally, by rearranging the terms and ignoring the constants, we obtain the required result.
			\end{proof}
			
			\section{Proof of Theorem \ref{althm}}\label{pfalthm}
			To prove almost sure convergence for the unconstrained version of \eqref{mainProb}, the coupled Supermartingale Convergence Theorem has been used in \cite[Theorem 5]{wang2017stochastic}. In the current context however, since we have not assumed anything on the boundedness of $\norm{\lam_{t}}^2$, the same cannot be used. Instead, we use different approach, wherein we add the various quantities in  \eqref{primopteq}, \eqref{dualopteq}, \eqref{trk1}, and \eqref{trk2}, and study the convergence of the resulting sequence. Then, by applying Supermartingale Convergence Theorem \cite{robbins1971convergence} to that cumulative sequence, we prove $\norm{\lam_{t}}$ is bounded for all $t$ with probability 1. Subsequently, we apply Supermartingale Convergence Theorem to each of the sequences individually and obtain the required result.   
			
			We begin by combining the statements of Lemmas \ref{primopt} and \ref{dualopt}, to obtain
			\begin{align}
			&\Ec{\norm{\x_{t+1} - \xt}^2 + \norm{\lam_{t+1} -\lt}^2  } \!\leq \!\br{\norm{\x_t - \xt}^2 +  \norm{\lam_t - \lt}^2}\!\br{1+ \frac{\alpha_t^2JC_{\ell}}{\beta_{t}}} +  \frac{\alpha_t^2}{\beta_{t}}L_f^2C_gD_x \nonumber \\
			&\hspace{0mm}- 2\alpha_t\br{\cL\br{\x_t,\lt,\alpha_{t},\delta_t} - \cL\br{\xt,\lam_t,\alpha_{t},\delta_t}} + \br{\frac{\alpha_t^2}{\beta_{t}}JL_{\ell}^2C_hD_x + 2\alpha_t^2JC_hC_{\ell} + 4\alpha_t^4\delta_t^2} \norm{\lam_t}^2 \nonumber \\
			&\hspace{0mm}+ 2\beta_{t}\br{1+4JC_{\ell}}\Ec{ \norm{\gb\br{\x_t} - \y_{t+1}}^2  +  \norm{\hb\br{\x_t} - \w_{t+1}}^2 } \!+ \br{2C_fC_g + 4JB_{\ell}+4J\theta^2}\alpha_t^2. \label{comb}
			\end{align}
			Let us define 
			\begin{align}
			\cI_{t} = \norm{\x_{t} - \xt}^2 + \norm{\lam_{t} -\lt}^2 + 2 \br{1+4JC_{\ell}} \br{\norm{\gb\br{\x_{t-1}} - \y_{t}}^2  +  \norm{\hb\br{\x_{t-1}} - \w_{t}}^2}. 
			\end{align}
			Multiplying \eqref{trk1}, \eqref{trk2} by $2 \br{1+6JC_{\ell}}\br{1+\beta_{t}}$ and adding with \eqref{comb}, we obtain
			\begin{align}
			&\Ec{\cI_{t+1} } \leq \cI_t\br{1+ \frac{\alpha_t^2JC_{\ell}}{\beta_{t}}} - 2\alpha_t\br{\cL\br{\x_t,\lt,\alpha_{t},\delta_t} - \cL\br{\xt,\lam_t,\alpha_{t},\delta_t}} +  \frac{\alpha_t^2}{\beta_{t}}L_f^2C_gD_x  \nonumber \\
			&\hspace{2mm}+ \br{\frac{\alpha_t^2}{\beta_{t}}JL_{\ell}^2C_hD_x + 2\alpha_t^2JC_hC_{\ell} + 4\alpha_t^4\delta_t^2} \norm{\lam_t}^2  +4\beta_{t}^{-1}\br{1+4JC_{\ell}}\br{C_g+C_h}\norm{\x_t-\x_{t-1}}^2\nonumber\\
			&\hspace{2mm} +  8\beta_t^2\br{1+4JC_{\ell}}\br{V_g + V_h} + \br{2C_fC_g + 4JB_{\ell}+4J\theta^2}\alpha_t^2 
			. \label{jb}
			\end{align}
			For the sake of brevity, let us define 
			\begin{align}
			C := \max&\Big\{JC_l, L_f^2C_gD_x, JL_l^2C_hD_x, JC_hC_l,4,4\br{1+4JC_l}\br{C_g+C_h}, \nonumber
			\\&8\br{1+4JC_{\ell}}\br{V_g + V_h}, 2C_fC_g + 4JB_{\ell}+4J\theta^2\Big\}
			\end{align}
			so that 
			\begin{align}\label{asfin}
			&\Ec{\cI_{t+1} } \leq \cI_t\br{1+ C\frac{\alpha_t^2}{\beta_{t}}} - 2\alpha_t\br{\cL\br{\x_t,\lt,\alpha_{t},\delta_t} - \cL\br{\xt,\lam_t,\alpha_{t},\delta_t}} \nonumber \\
			&\hspace{2mm} +  C\br{\frac{\alpha_t^2}{\beta_{t}} +  \beta_t^2+ \alpha_t^2} + C\br{\frac{\alpha_t^2}{\beta_{t}} + \alpha_t^2 + \alpha_t^4\delta_t^2} \norm{\lam_t}^2 + C\beta_{t}^{-1}\norm{\x_t-\x_{t-1}}^2.
			\end{align}
			From \eqref{confac}, we can write
			\begin{align*}
			\cL\br{\x_t,\lt,\alpha_{t},\delta_t} - \cL\br{\xt,\lam_t,\alpha_{t},\delta_t} = F(\x_{t}) + \sum_{j= 1}^J\lt_jL_j(\x_{t}) - F(\xt) + \frac{\alpha_{t}\delta_t}{2}\br{\norm{\lam_{t}}^2 - \norm{\lt}^2}
			\end{align*}
			which by substituting back in \eqref{asfin}, we get
			\begin{align}\label{fesfor}
			&\Ec{\cI_{t+1} } \leq \cI_t\br{1+ C\frac{\alpha_t^2}{\beta_{t}}} - 2\alpha_t\br{F(\x_{t}) + \sum_{j= 1}^J\lt_jL_j(\x_{t}) - F(\xt)} + \alpha_{t}^2\delta_t\norm{\lt}^2  \nonumber \\
			&\hspace{2mm} + C\br{\frac{\alpha_t^2}{\beta_{t}} +  \beta_t^2+ \alpha_t^2} + \br{\frac{C\alpha_t^2}{\beta_{t}} + C\alpha_t^2 + C\alpha_t^4\delta_t^2 - \alpha_{t}^2\delta_t} \norm{\lam_t}^2 + C\beta_{t}^{-1}\norm{\x_t-\x_{t-1}}^2.
			\end{align}
			The result in \eqref{fesfor} holds for any feasible point. Note that, from the statement of Theorem \ref{althm}, it is given  $\theta = 0$ in \eqref{sur}. Hence, both of the problems in \eqref{mainProb} and \eqref{sur} are equivalent, which implies any feasible point to \eqref{sur} is also feasible to \eqref{mainProb}. Hence we replace $\br{\xt,\lt}$ by the saddle point $\br{\xs,\vls}$ which is the optimal solution pair of primal problem in \eqref{mainProb}, and dual problem in \eqref{dualProb}
			\begin{align}
			&\Ec{\cI_{t+1} } \leq \cI_t\br{1+ C\frac{\alpha_t^2}{\beta_{t}}}  + \br{C\frac{\alpha_t^2}{\beta_{t}} + C\alpha_t^2 + C\alpha_t^4\delta_t^2 - \alpha_{t}^2\delta_t} \norm{\lam_t}^2 + C\beta_{t}^{-1}\norm{\x_t-\x_{t-1}}^2 \nonumber \\
			&\hspace{2mm} - 2\alpha_t\br{F(\x_{t}) + \sum_{j= 1}^J\ls_jL_j(\x_{t}) - F(\xs)} + \alpha_{t}^2\delta_t\norm{\vls}^2 +C\br{\frac{\alpha_t^2}{\beta_{t}} +  \beta_t^2+ \alpha_t^2}. 
			\end{align}
			
			Defining $\cJ_t := D\beta_{t}^{-1}\norm{\x_t-\x_{t-1}}^2$, and $\cK_t := \cI_t + \cJ_t$, and rearranging the terms, we can write,
			\begin{align}
			&\Ec{\cK_{t+1} } \nonumber\\
			&\hspace{2mm}\leq \cK_t\br{1+ C\frac{\alpha_t^2}{\beta_{t}}}  + \br{C\frac{\alpha_t^2}{\beta_{t}} + C\alpha_t^2 + C\alpha_t^4\delta_t^2 - \alpha_{t}^2\delta_t} \norm{\lam_t}^2 + C\beta_{t+1}^{-1}\Ec{\norm{\x_{t+1}-\x_{t}}^2} \nonumber \\
			&\hspace{10mm} - 2\alpha_t\br{F(\x_{t}) + \sum_{j= 1}^J\ls_jL_j(\x_{t}) - F(\xs)} + \alpha_{t}^2\delta_t\norm{\vls}^2 +C\br{\frac{\alpha_t^2}{\beta_{t}} +  \beta_t^2+ \alpha_t^2} 
			\end{align}
			\begin{align}
			&\hspace{2mm}\leq \cK_t\br{1+ C\frac{\alpha_t^2}{\beta_{t}}}  + \br{C\frac{\alpha_t^2}{\beta_{t}} + C\alpha_t^2 + C\alpha_t^4\delta_t^2 - \alpha_{t}^2\delta_t} \norm{\lam_t}^2 + \alpha_{t}^2\delta_t\norm{\vls}^2 +C\frac{\alpha_t^2}{\beta_{t}}  + C\alpha_t^2 \nonumber \\
			& \hspace{10mm} +  C\beta_t^2- 2\alpha_t\br{F(\x_{t}) + \sum_{j= 1}^J\ls_jL_j(\x_{t}) - F(\xs)} + 2C\frac{\alpha_{t}^2}{\beta_{t+1}}\br{C_fC_g+JC_hC_l\norm{\lam_{t}}^2} ,
			\end{align}
			where the inequality follows from \eqref{iteb}. Now, by again defining $D = \max\left\{C,2CC_fC_g,2CJC_hC_l\right\}$, we write 
			\begin{align}
			& \Ec{\cK_{t+1} } \leq \cK_t\br{1+ D\frac{\alpha_t^2}{\beta_{t}} } + \underbrace{\br{D\frac{\alpha_t^2}{\beta_{t}} + D\alpha_t^2 + D\alpha_t^4\delta_t^2 + D\frac{\alpha_{t}^2}{\beta_{t+1}} - \alpha_{t}^2\delta_t} \norm{\lam_t}^2}_{\nu_t} \nonumber \\
			&\hspace{2mm} - 2\alpha_t\br{F(\x_{t}) + \sum_{j= 1}^J\ls_jL_j(\x_{t}) - F(\xs)} + \alpha_{t}^2\delta_t\norm{\vls}^2 +D\br{\frac{\alpha_t^2}{\beta_{t}} +   \frac{\alpha_{t}^2}{\beta_{t+1}}+  \beta_t^2+ \alpha_t^2}. \label{vterm}
			\end{align}
			Observe that $\nu_t$ is a convex quadratic function of $\delta_t$ and from Lemma \ref{chLem}, it can be made non-positive, if we choose 
				\begin{align*}
				\delta_t = 2K\br{1 + \frac{1}{\beta_{t}} + \frac{1}{\beta_{t+1}}}.
				\end{align*} 
				Hence $\alpha_{t}^2\delta_t = \O\br{\alpha_{t}^2+\alpha_{t}^2/\beta_{t} + \alpha_{t}^2/\beta_{t+1}}$. Now by dropping $\nu_t$ in \eqref{vterm} and substitute $\delta_t$, we obtain
			\begin{align}
			\Ec{\cK_{t+1} } &\leq \cK_t\br{1+ \underbrace{D\frac{\alpha_t^2}{\beta_{t}}}_{\eta_t} } - \underbrace{2\alpha_t\br{F(\x_{t}) + \sum_{j= 1}^J\ls_jL_j(\x_{t}) - F(\xs)}}_{u_t} \nonumber \\
			&\hspace{2mm}  + \underbrace{\O\br{\alpha_{t}^2+\frac{\alpha_{t}^2}{\beta_{t}} + \frac{\alpha_{t}^2}{\beta_{t+1}}}\norm{\vls}^2 +D\br{\frac{\alpha_t^2}{\beta_{t}} +   \frac{\alpha_{t}^2}{\beta_{t+1}}+  \beta_t^2+ \alpha_t^2}}_{\mu_t} \label{asexp}.
			\end{align}
			
			Since $\br{\xs,\vls}$ is the saddle point, by KKT conditions,  we have that
			\begin{align}
			\xs &= \arg\min_{\x\in\cX} F(\x) + \sum_{j= 1}^J\ls_jL_j(\x) \label{kktpr}\\
			\vls &= \arg\max_{\lam \geq \vzero} F(\xs) +\sum_{i=1}^J\lamb_jL_j(\xs), \label{kktdu}
			\end{align}
			or equivalently for any  $\x \in \cX$ and $\lam \geq \vzero$,
			\begin{align}\label{kktult}
			F(\x) + \sum_{j= 1}^J\ls_jL_j(\x) \geq F(\xs) + \sum_{j= 1}^J\ls_jL_j(\xs) \geq F(\xs) +\sum_{i=1}^J\lamb_jL_j(\xs).
			\end{align}
			Since $ \sum_{j= 1}^J\ls_jL_j(\xs) = 0$ due to the complementary slackness condition, the sequence $u_t$ in \eqref{asexp} is non negative. Recall from Assumption \ref{aslater} that $\norm{\vls}$ is bounded. Therefore from the step-size choices made in the statement of Theorem \ref{althm}, the sequences $\eta_t$ and $\mu_t$ are summable with probability 1. Applying the Supermartingale Convergence Theorem to \eqref{asexp}, we can say the $\cK_t$ converges almost surely to a nonnegative random variable, and $\sum_{t=1}^\infty u_t < \infty$ with probability 1. Therefore we have that
			\begin{align}
			\lim_{t\rightarrow\infty}\inf \bc{F(\x_t) +\sum_{j= 1}^J\ls_jL_j(\x_t)} = F(\xs) \hspace{0.5cm} w.p.1. 
			\end{align} 
			Since $\cK_t$ converges almost surely, the sequences $\cI_t$, $\cJ_t$ must be bounded with probability 1, and consequently, the component sequences $\norm{\x_t-\xs}^2$, $\norm{\lam_{t}-\vls}^2$, $\norm{\y_t - \bar{g}(\x_{t-1})}^2$, and $\norm{\w_t - \bar{h}(\x_{t-1})}^2$ are also bounded with probability 1. Next, observe that $\cJ_t$ can be written as 
			\begin{align}
			\Ec{\cJ_{t+1}} &=\frac{D}{\beta_{t+1}}\norm{\x_{t+1}-\x_t}^2= 2D\frac{\alpha_{t}^2}{\beta_{t+1}}\br{C_gC_f + JC_hC_l\norm{\lam_{t}}^2} \\
			&\leq 2D\frac{\alpha_{t}^2}{\beta_{t+1}}\br{C_gC_f + 2JC_hC_l\norm{\lam_{t} - \vls}^2 + 2JC_hC_l\norm{\vls}^2} ,
			\end{align} 
			where note that the sequence $\alpha_{t}^2/\beta_{t+1}$ is summable and  $\norm{\vls}^2$ is finite. Therefore, from the Supermartingale Convergence Theorem, it follows that $\cJ_t$ is almost surely convergent to a non negative random variable. Further, the Monotone Convergence Theorem implies that $\sum_{t=1}^\infty\beta_{t}^{-1}\norm{\x_{t}-\x_{t-1}}^2 < \infty$ with probability one.  Now consider
			\begin{align}
			\Ec{\norm{\y_{t+1} - \gb\br{\x_t}}^2 } &\leq \br{1 - \beta_t}\norm{\y_t - \gb\br{\x_{t-1}}}^2  + \beta_t^{-1}C_g\norm{\x_{t}-\x_{t-1}}^2 +2V_g\beta_t^2, \label{trk1as}\\
			\Ec{\norm{\w_{t+1} - \hb\br{\x_t}}^2} &\leq \br{1 - \beta_t}\norm{\w_t - \hb\br{\x_{t-1}}}^2 + \beta_t^{-1}C_h\norm{\x_{t}-\x_{t-1}}^2+2V_h\beta_t^2. \label{trk2as}
			\end{align} 
			since $\beta_t^2$ and $\beta_{t}^{-1}\norm{\x_{t}-\x_{t-1}}^2$ are summable, it follows from the almost Supermartingale Convergence Theorem that $\norm{\y_t - \gb\br{\x_{t-1}}}^2$ and $\norm{\w_t - \hb\br{\x_{t-1}}}^2$ are almost surely convergent to some non negative random variables, and further we have
			\begin{align}
			\sum_{t=1}^\infty \beta_{t}\norm{\y_{t+1} - \gb\br{\x_{t}}}^2 &< \infty \hspace{0.5cm} w.p.1, \\
			\sum_{t=1}^\infty \beta_{t}\norm{\w_{t+1} - \hb\br{\x_{t}}}^2 &< \infty \hspace{0.5cm} w.p.1, 		
			\end{align} 
			which also implies that
			\begin{align}\label{liminf}
			\lim_{t\rightarrow\infty}\inf \norm{\y_{t+1} - \gb\br{\x_{t}}}^2 = 0 \hspace{0.5cm} w.p.1, \hspace{1cm}
			\lim_{t\rightarrow\infty}\inf \norm{\w_{t+1} - \hb\br{\x_{t}}}^2 = 0 \hspace{0.5cm} w.p.1.
			\end{align}
			Since $\cI_t$, $\cJ_t$, $\norm{\y_t - \gb\br{\x_{t-1}}}^2$,  $\norm{\w_t - \hb\br{\x_{t-1}}}^2$ are almost surely convergent, we can conclude $\norm{\x_t - \xs}^2 + \norm{\lam_{t}-\vls}^2$ is also almost surely convergent to a non negative random variable. Let $\chi_t = \br{\x_t,\lam_{t}}$ and $\chis = \br{\xs,\vls} \in \cXs$, where $\cXs$ is the set of primal and dual optimal pairs. We have established that for any $\chis \in \cXs$, the sequence $\norm{\chi_t - \chis}$ converge almost surely to a non negative random variable. for any $\chis \in \cXs$, let $\upsilon$ be a sample path such that $\lim_{t\rightarrow\infty} \norm{\chi_t(\upsilon) - \chis}$ exists, and $\Omega_{\chis}$ is the collection of all such paths $\upsilon$. For any $\chis \in \cXs$, It implies that $\Pn\br{\Omega_{\chis}} = 1$. 
			
			The following lemma establishes that, the sequence converges almost surely to any optimal solution.
			\begin{lemma}\label{asmat}
				Let $\vchi_t = \br{\x_t,\lam_{t}}$. Then $\norm{\vchi_t - \tilde{\vchi}}$ is convergent for all $\tilde{\vchi}\in\cXs$ with probability 1.
			\end{lemma} 
			The proof is provided in \cite[Theorem 1(a), page 13-14]{wang2017stochastic}. 
			
			Before concluding the proof we explain the implication of the statements of Lemma \ref{asmat}. Consider an arbitrary sample trajectory of $\bc{\x_t(\upsilon)}$ such that $\upsilon \in \bigcap_\cXs\Omega_{\chis}$ and $\lim_{t\rightarrow\infty}\inf F(\x_t) + \sum_{j= 1}^J\ls_jL_j(\x_t) = F(\xs)$. For any arbitrary $\br{\xs,\vls} \in \cXs$, since $\norm{\br{\x_t(\upsilon), \lam_t(\upsilon)} - \br{\xs,\vls}}$ converges, the sequence is bounded. By continuity of $F, L_j$, the sequence $\bc{\x_t(\upsilon)}$ must have a limit point $\bar{\x}$ such that $F(\bar{\x}) + \sum_{j= 1}^J\ls_jL_j(\bar{\x}) = F(\xs)$. Observe from \eqref{kktpr}, it is clear that, $\xs$ is the minimizer of $F(\x) + \sum_{j=1}^J\ls_jL_j(\x)$ of any $\x \in \cX$. Hence $F(\bar{\x}) + \sum_{j= 1}^J\ls_jL_j(\bar{\x}) = F(\xs)$ occurs only at $\bar{\x} = \xs$ which is an optimal solution of \eqref{mainProb}. Since $\upsilon \in \bigcap_\cXs\Omega_{\chis} \subset \Omega_{\br{\bar{\x},\vls}}$, we can say $\bc{\norm{\x_t(\upsilon) - \bar{\x}}}$ is also a convergent sequence. Since it is convergent with limit point 0, we have $\norm{\x_t(\upsilon) - \bar{\x}} \rightarrow 0 $. Then $\x_t(\upsilon)\rightarrow \bar{\x}$ on this sample trajectory. Note that set of all sample paths has a probability measure of one. Hence we complete the proof by concluding, $\x_t$ almost surely converges to a random point in the set of optimal solutions of \eqref{mainProb}.
			
			\section{Proof of Theorem \ref{bcthm}}\label{pfbcthm}
			The  proof is divided into two parts. In the first part we choose $\lt = \vzero$, and $\xt = \x^\theta$ to prove convergence of objective error. Later, we choose $\lt = \lam^{\theta}$, and $\xt = \x^\theta$ to prove convergence of constraint violation. As a first step we analyze the optimality gap. Recall that from Lemma \ref{dualopt}, the results here are hold for any $\lt \geq 0$, and feasible point $\xt \in \cX$ s.t $L_j(\xt) +\theta \leq 0$ $\forall j$. Since $\x^\theta$ is also a feasible point to the problem \eqref{sur}, for the first result, we choose $\lt =0$ and $\xt = \x^\theta$. Substituting into \eqref{lams}, we obtain
			\begin{align}
			\sum_{t=1}^T\alpha_{t}\E{F\br{\x_t} - F\br{\x^{\theta}}} \leq \omega T.\label{simpco}
			\end{align}
			Let $\cA = \sum_{t=1}^{T}\alpha_{t}$ and for large $T$, we can approximate $\cA$ as $\cO\br{T^{1-a}}$ when $a\neq1$. Since $F$ is a convex function we can write
				\begin{align}
				&\E{F(\hat{\x}) - F(\xs)} \leq \frac{1}{\cA}\sum_{t=1}^T\alpha_{t}\E{F\br{\x_t} - F\br{\xs}} \nonumber\\
				&\hspace{2mm} \leq \omega + \frac{1}{\cA}\sum_{t=1}^T\alpha_{t}\E{F\br{\x^\theta} - F\br{\xs}}\label{rest} \\
				&\leq \omega + \theta\frac{2\sqrt{C_fC_g}D_x}{\sigma_0}\label{restfin} 
				\end{align}
			where the first step in \eqref{rest} is obtained by adding and subtracting $F\br{\x^\theta}$, and the final step in \eqref{restfin} follows from Lemma \ref{zc}. 
			
			Next, we choose $\lt = \lam^{\theta}$ which is dual optimum of the problem \eqref{sur}, i.e.,
			\begin{align}
			\lam^{\theta} = \arg \max_{\lam \geq \vzero}\bc{\min_{\x \in \cX}F(\x) + \sum_{j=1}^J \lambda_j\br{L_j(\x) +\theta}}.
			\end{align}
			Since $\x^\theta$ is the optimal solution of primal problem in \eqref{sur}, and $\theta \leq \sigma_0/2$, from assumption \ref{aslater}, the pair $\br{\x^\theta,\lam^{\theta}}$ constitutes a saddle point of the problem \eqref{sur}, and satisfies KKT conditions
			\begin{align}
			\x^{\theta} &= \arg\min_{\x\in \cX} F(\x) + \sum_{j=1}^J\lamb_j^{\theta}\br{L_j(\x) + \theta} \label{chsxstr},\\
			\lam^{\theta} &= \arg\max_{\lam\geq \vzero} F\br{\x^{\theta}} + \sum_{j=1}^J\lamb_j\br{L_j(\x^{\theta}) + \theta} \label{chsvls}.
			\end{align} 
			or equivalently, for any $\x \in \cX$, and $\lam \geq 0$, 
			\begin{align}
			F(\x) + \sum_{j=1}^J\lamb_j^{\theta}\br{L_j(\x) + \theta} \geq F\br{\x^{\theta}} + \sum_{j=1}^J\lamb_j^{\theta}\br{L_j(\x^{\theta}) + \theta} \geq F\br{\x^{\theta}} + \sum_{j=1}^J\lamb_j\br{L_j(\x^{\theta}) + \theta} \label{obs}.
			\end{align}
			For the ease of analysis, we write the classical definition of Lagrangian function for the problem in \eqref{sur} as
			\begin{align}
			\Lh^{\theta} (\x, \lam) = F(\x) + \sum_{j=1}^J\lamb_j\br{L_j(\x) + \theta}.
			\end{align}
			Let $\vone^i$ be a vector in which $i^{\text{th}}$ entry is unity and zeros elsewhere. Then
			\begin{align}
			\E{\Lh^{\theta} (\x_t, \vone^i+\lam^{\theta})} &= \E{F(\x_t)} + \E{L_i(\x_t) + \theta + \sum_{j= 1}^J\lamb_j^{\theta}\br{L_j(\x_t) + \theta}} \\
			& = \E{\Lh^{\theta} (\x_t,\lam^{\theta})} + \E{L_i(\x_t)} + \theta
			\end{align}
			By rearranging the terms, we write
			\begin{align}
			\E{L_i(\x_t)} + \theta = \E{\Lh^{\theta} (\x_t, \vone^i+\lam^{\theta}) - \Lh^{\theta} (\x_t,\lam^{\theta}) }
			\end{align}
			From the observation in \eqref{obs}, we can say $\Lh^{\theta} (\x_t,\lam^{\theta}) \geq \Lh^{\theta} (\x^\theta,\lam^{\theta})$. Hence we can write 
			\begin{align}\label{rew}
			\E{L_i(\x_t)} + \theta &\leq \E{\Lh^{\theta} (\x_t, \vone^i+\lam^{\theta}) - \Lh^{\theta} (\x^{\theta}, \lam^{\theta})} \\
			&\leq \E{F(\x_t) + L_i(\x_t) + \theta +\sum_{j=1}^J \lamb_j^{\theta}\br{L_j(\x_t) + \theta} } - F(\x^{\theta}).
			\end{align}
			The second inequality follows from the complementary slackness condition \\
			$\sum_{j=1}^J\lamb_j^{\theta}\br{L_j(\x^{\theta}) + \theta} = 0$.
			Since the result in Lemma  \ref{lagl} holds for any $\lt \geq 0$, we can replace $\lt$ by $\vone^i+\lam^{\theta}$. Hence by summing over $1,...,T$, we can rewrite \eqref{rew} as
			\begin{align*}
				\E{L_i(\hat{\x})} &+ \theta \leq  \frac{1}{\cA}\sum_{t=1}^T\alpha_{t}\E{L_i(\x_t)} + \theta \leq\omega\br{1+ \norm{\vone^i + \lam^{\theta}}^2} \leq \omega\br{2+ \frac{8C_fC_gD_x^2}{\sigma_0^2}}
				\end{align*}
			where the last step follows from the bound on $\norm{\lam^\theta}$ in Lemma \ref{zc}. Finally the required order bound on $\omega$ can be obtained simply by substituting the constant/diminishing step-sizes and ignoring all constant terms.


%
%


\footnotesize
	\bibliographystyle{IEEEtran}
	\bibliography{IEEEabrv,references}	

\end{document}